\documentclass[journal]{IEEEtran}

\ifCLASSINFOpdf
\else
\fi

\usepackage{amsthm}
\usepackage{amsfonts,amssymb,latexsym,cite}
\usepackage[cmex10]{amsmath}
\usepackage{algorithmic}
\usepackage{array}
\usepackage{bbm}
\usepackage{mathrsfs}
\usepackage{multirow}
\usepackage{verbatim}
\usepackage{color,xcolor}
\usepackage{graphicx}
\usepackage{epstopdf}
\usepackage{subcaption}
\usepackage[font={small,it}]{caption}
\usepackage[T1]{fontenc}
\usepackage{url}
\usepackage{enumerate}
\usepackage{hyperref} 
\graphicspath{{./figures/}} 
\usepackage{caption}
\usepackage{stfloats} 
\usepackage[ruled,vlined]{algorithm2e}
\usepackage[blocks]{authblk}
\usepackage{epsfig}
\usepackage{latexsym}
\usepackage{soul}
\usepackage{booktabs} 
\usepackage{enumitem,kantlipsum}

\newtheorem{lemma}{Lemma}
\newtheorem{definition}{Definition}
\newtheorem{theorem}{Theorem}

\newtheorem{remark}{Remark}

\newcommand{\upperRomannumeral}[1]{\uppercase\expandafter{\romannumeral#1}}

\usepackage{lipsum}

\newcommand{\U}{\mathcal{U}}
\newcommand{\UU}{\mathsf{U}}
\newcommand{\I}{\mathcal{I}}
\newcommand{\N}{\mathsf{N}}

\newcommand{\tl}{\widetilde{L}}
\newcommand{\bE}{\bar{\mathcal{E}}}

\newcommand{\R}{\mathsf{B}}

\newcommand{\PP}{\mathbb{P}}
\newcommand{\E}{\mathbb{E}}

\newcommand{\qa}{Q_0^\mathrm{a}}
\newcommand{\qaa}{Q_1^\mathrm{a}}
\newcommand{\qr}{Q_0^\mathrm{r}}
\newcommand{\qrr}{Q_1^\mathrm{r}}

\newcommand{\s}{\mathcal{S}}
\newcommand{\T}{\mathcal{T}}

\newcommand{\hqab}{\widehat{Q}_{ab}}
\newcommand{\hqabb}{\widehat{Q}_{a'b}}

\newcommand{\ta}{\theta_{\mathcal{A}}}
\newcommand{\tr}{\theta_{\mathcal{R}}}

\newcommand{\ba}{\bar{a}}
\newcommand{\bb}{\bar{b}}

\newcommand{\darm}{d_{\mathcal{AR}}}

\newcommand{\mc}[1]{{\color{blue}#1}}

\DeclareMathOperator*{\argmax}{arg\,max}

\allowdisplaybreaks[3]

\begin{document}
%
\title{{\sc Mc2g}: An Efficient Algorithm for Matrix \\ Completion with Social and Item Similarity Graphs}
%
%
%

\author{Qiaosheng Zhang, Geewon Suh,
	Changho Suh,~\IEEEmembership{Senior~Member,~IEEE,}
	Vincent Y.~F.~Tan,~\IEEEmembership{Senior~Member,~IEEE}
	\thanks{Q. Zhang and G. Suh contributed equally to this work.}
	\thanks{Q. Zhang and V.~Y.~F.~Tan are with the Department of Electrical and Computer Engineering, National University of Singapore (Emails: elezqiao@nus.edu.sg, vtan@nus.edu.sg). V.~Y.~F.~Tan is also with the Department of Mathematics, National University of Singapore. G. Suh and C. Suh are with the School of Electrical Engineering, KAIST (Email: gwsuh91@kaist.ac.kr, chsuh@kaist.ac.kr). }  }

%
%

\markboth{}%
{Shell \MakeLowercase{\textit{et al.}}: Bare Demo of IEEEtran.cls for IEEE Journals}
%



\maketitle


\begin{abstract}
In this paper, we design and analyze  {\sc Mc2g} (\underline{M}atrix \underline{C}ompletion with \underline{2} \underline{G}raphs), an algorithm that performs matrix completion in the presence of social and item similarity graphs. {\sc Mc2g} runs in quasilinear time and is parameter free. It is based on spectral clustering and local refinement steps.  
The expected number of sampled entries required for {\sc Mc2g} to succeed (i.e., recover the clusters in the graphs and complete the matrix) matches an information-theoretic lower bound up to a constant factor  for a wide range of parameters.  We show via extensive experiments on both synthetic and real datasets that {\sc Mc2g} outperforms other state-of-the-art matrix completion algorithms that leverage graph side information.
 
\end{abstract}


\section{Introduction} \label{sec:introduction}

With the ubiquity of social networks such as Facebook and Twitter, it is increasingly convenient to collect similarity information amongst users. It has been shown that exploiting this similarity information in the form of a \emph{social graph} can significantly improve the quality of recommender systems~\cite{aggarwal1999horting,massa2007trust, ma2011recommender,ahn2018binary,jo2020discrete, kalofolias2014matrix} compared to traditional recommendation algorithms (e.g., collaborative filtering~\cite{goldberg1992using, sarwar2001item}) that rely merely on rating information. This improvement is particularly pronounced in the presence of the so-called \emph{cold-start problem} in which we would like to recommend items to a user who has not rated any items, but we do possess his/her similarity information with other users. Similarly, an \emph{item similarity graph} is sometimes also available for exploitation---it can be constructed either from the features of items~\cite{condliff1999bayesian,rattigan2007graph}, or from users' behavior history (as has been done by Taobao~\cite{wang2018billion}). Again, this can help in solving the \emph{dual cold-start problem}, namely the learner has no information  about new items that have not been rated by any user.

While there have been numerous studies considering how to exploit graph side information to enhance recommender systems, most of the algorithms developed so far exploit \emph{only one} graph (either the social or the item similarity graph). As mentioned above, \emph{both} graphs are often available in many real-life applications, and it has been shown in a prior theoretical study~\cite{eric2021} that there are scenarios in which exploiting two graphs yields strictly more benefits than exploiting only one graph. This work builds upon~\cite{eric2021} which focuses on {\em fundamental limits}, but  does not propose computationally efficient algorithms that achieve the limits. Our main contribution is to design and analyze  a computationally efficient algorithm---which we name {\sc Mc2g}---for a matrix completion problem, wherein both the social and item similarity graphs are available.  We also provide theoretical  guarantees on the expected number of sampled entries for {\sc Mc2g} to succeed, and further show that it meets an information-theoretic lower bound up to a constant factor. It is worth highlighting that {\sc Mc2g}  is  applicable to a more general setting than that considered  in~\cite{eric2021}, thus the theoretical results developed in this work further generalize the theory in~\cite{eric2021}.  For example, we consider general discrete-valued ratings instead of binary ratings.

We consider a setting in which there are $n$ users and $m$ items. Users are partitioned into $k_1 \ge 2$ clusters, while items are partitioned into $k_2 \ge 2$ clusters.  Users' ratings to items are chosen from an \emph{arbitrarily} pre-assigned finite set (e.g., a reasonable choice is $\{1,2,3,4,5\}$, which models the Netflix prize challenge~\cite{bennett2007netflix}).
The $n \times m$ rating matrix is generated  according to a generative model which we describe  in Section~\ref{sec:model}. The learner observes three pieces of information: (i) a sub-sampled rating matrix with each entry being sampled independently with probability $p$; (ii) a social graph generated according to a celebrated generative model for random graphs---the \emph{stochastic block model} (SBM)~\cite{holland1983stochastic}; and (iii) an item similarity graph generated according to another SBM. The task is to exactly recover the clusters of both users and items, as well as to complete the matrix.
Our model significantly generalizes the models considered in several related works with theoretical guarantees~\cite{ahn2018binary,jo2020discrete, eric2021}, by relaxing some constraints therein, e.g., (i) users/items are only partitioned into two equal-sized clusters, and (ii) only binary ratings are allowed.

\subsection{Main contributions} 
Our main contributions are summarized as follows.

\begin{enumerate}[wide, labelwidth=!, labelindent=0pt]
\item We develop a computationally efficient algorithm {\sc Mc2g} that runs in \emph{quasilinear time}. {\sc Mc2g} is a multi-stage algorithm that follows the ``\emph{from global to local}'' principle\footnote{This principle is not only applicable to matrix completion~\cite{keshavan2010matrix,jain2013low}, but has also been applied to many other problems such as community detection~\cite{abbe2015community, abbe2017community,gao2017achieving}, phase retrieval~\cite{candes2015phase,netrapalli2015phase}, etc.}---it first adopts a \emph{spectral clustering method} on graphs to obtain initial estimates of user/item clusters, and then refines each user/item individually based on {\em local} maximum likelihood estimation (MLE). {\sc Mc2g} is also a parameter-free algorithm that does not need the knowledge of the model parameters.  Under the symmetric setting wherein both the social and item similarity graphs are generated according to symmetric SBMs~\cite[Def.~2]{abbe2017community}, we show that {\sc Mc2g} succeeds in the sense of recovering the missing entries of the sub-sampled matrix and the clusters with high probability as long as the number of samples exceeds a bound presented in Theorem~\ref{thm:algorithm}. While the theoretical guarantee requires the symmetric assumption, we emphasize that {\sc Mc2g} is universally applicable to all matrix completion problems with two-sided graph side information.

\item  We also provide an information-theoretic lower bound that matches the bound in Theorem~\ref{thm:algorithm} up to a constant factor; this demonstrates the order-wise optimality of {\sc Mc2g}. As a by-product, the aforementioned theoretical results also generalize the theory developed in the prior work~\cite{eric2021}, which was focused on a simpler setting in which both users and items are partitioned into two clusters. 

\item We conduct extensive experiments on synthetic datasets to verify that the results show keen agreement with the derived theoretical guarantee of {\sc Mc2g} in Theorem~\ref{thm:algorithm}. We further demonstrate the superior performance of {\sc Mc2g} by comparing it with several state-of-the-art matrix completion algorithms that leverage graph side information, such as matrix factorization with social regularization (SoReg)~\cite{ma2011recommender}, and a spectral clustering method with local refinements using {\em only} the social graph or only the item graph~\cite{ahn2018binary}. {\sc Mc2g} is often orders of magnitude better than the competing algorithms in terms of the mean absolute error (MAE).

\item  Finally we apply {\sc Mc2g} to datasets with \emph{real} social and item similarity graphs (i.e., the LastFM  social network~\cite{rozemberczki2020characteristic} and political blogs network~\cite{adamic2005political}). Our experimental results show that {\sc Mc2g} works well when the observed graphs are derived from real-world applications; this further confirms that {\sc Mc2g} is universal, as the real graphs do not satisfy the symmetry assumptions. Finally, we compare {\sc Mc2g} with the other aforementioned matrix completion algorithms on the dataset with real graphs. Experimental results show that {\sc Mc2g} outperforms these existing algorithms.  
\end{enumerate}

\subsection{Related works}      

Due to the wide applicability of matrix completion (such as recommender systems), the past decade has witnessed the developments of many efficient matrix completion algorithms, such as~\cite{candes2010power,marjanovic2012l_q,chen2015signal,dai2011subspace,ma2014decomposition}.  In the context of recommender systems, the design of algorithms that exploit graph side information (especially the social graph) has attracted much attention, and some of the works~\cite{kalofolias2014matrix,monti2017geometric,wang2019neural} exploit both the social and item similarity graphs. Although these algorithms usually have better empirical performance than traditional ones, most of them neither quantify the gains of exploiting graph side information, nor provide any theoretical guarantees.

We note that another line  of works focused on characterizing the fundamental limits of  matrix completion in which the matrix to be recovered is generated according to a  certain generative model for the clusterings of the users and/or items. Ahn \emph{et al.}~\cite{ahn2018binary} considered a simple setting where ratings are binary and a graph encodes the structure of two clusters, and characterized the expected number of sampled entries required for the matrix completion task. Follow-up works~\cite{jo2020discrete, yoon2018joint} relaxed the assumptions in~\cite{ahn2018binary}, but are still restricted to exploiting the use of a social graph. The recent work~\cite{eric2021}  investigated a more general setting in which both the social and item similarity graphs are available, and quantified the gains of exploiting two graphs  by establishing information-theoretic lower and upper bounds. However, a computationally efficient algorithm that achieves the limit promised by MLE was not developed in~\cite{eric2021}. Given that the MLE is not computationally feasible, there is a pressing need to develop efficient algorithms. This precisely sets the goal of this work. Additionally, this work studies a  generalized model that spans multiple user/item clusters and discrete-valued rating matrices. This is in contrast to~\cite{eric2021} which focuses on two clusters and binary ratings.

Another field relevant to this work is \emph{community detection}, which is the problem of partitioning nodes of  an undirected graph into different clusters/communities. When the graphs are generated according to SBMs, the information-theoretic limits for exact recovery of clusters~\cite{abbe2015exact,mossel2015reconstruction, abbe2015community, abbe2017community, zhang2016minimax, hajek2017information} have been established. These limits also play a role in establishing the theoretical guarantee of {\sc Mc2g} (see the third item in Remark~\ref{remark:thm} for details), as our algorithm includes the clustering step for users and items in the process of matrix completion. It has also been shown that side information is in general helpful for community detection~\cite{saad2018community, saad2018exact, mayya2019mutual,asadi2017compressing}. This observation is pertinent and related to our work because our problem can also be viewed as recovering users/items clusters with side information in the form of a rating matrix. Besides, our problem is also related to the labelled or weighted SBM problem, if the two SBMs that govern the social and item similarity graphs are merged to a single unified SBM (interested readers are referred to~\cite[Remark 4]{eric2021} for details).

\subsection{Outline}
This paper is organized as follows. We first introduce the problem setup in Section~\ref{sec:model}, and then describe our efficient algorithm {\sc Mc2g} in Section~\ref{sec:algorithm}. Section~\ref{sec:result} presents our main theoretical results: (i) the theoretical  guarantee for {\sc Mc2g} and (ii) the information-theoretical lower bound. These results are proved in Sections~\ref{sec:proof} and~\ref{sec:proof2}, respectively. Experimental results are presented in Section~\ref{sec:experiment}.

\section{Problem statement} \label{sec:model}

We consider a recommender system with $n$ users and $m$ items. Ratings from users to items are chosen from an arbitrary finite alphabet $\mathcal{Z}$ (e.g., $\mathcal{Z} = \{1,2,3,4,5\}$). It is assumed that users are partitioned into $k_1 \ge 2$ disjoint clusters $\left\{\U_1,\U_2, \ldots,\U_{k_1}\right\}$, and items are partitioned into $k_2\ge 2$ disjoint clusters $\left\{\I_1,\I_2, \ldots,\I_{k_2}\right\}$. We define\footnote{For any integer $s \ge 1$, let $[s]$ represent the set of integers $\{1,\ldots,s\}$. } $\sigma: [n] \to [k_1]$ as the \emph{label function for users} such that $\sigma(i) = a$ if user $i$ belongs to cluster $\U_a$. On the contrary, each clutser $\U_a$ can be represented as $\U_a = \{i\in[n]: \sigma(i) = a\}$. Thus, $\sigma$ can be viewed as an alternative (and more concise) representation of the clusterings of users $\{\U_{a} \}_{a \in [k_1]}$. Similarly, we define  $\tau: [m] \to [k_2]$ as the \emph{label function for items} such that $\tau(j) = b$ if item $j$ belongs to cluster $\I_b$.

\begin{table}
	\caption{Nominal ratings from users to items}
	\label{table:1}
	\centering
	\begin{tabular}{lllll}
		\toprule
		& Cluster $\I_1$      & Cluster $\I_2$ & $\ldots \ldots$ & Cluster $\I_{k_2}$   \\
		\midrule
		Cluster $\U_1$ & \multicolumn{1}{c}{$z_{11}$} & \multicolumn{1}{c}{$z_{12}$} & \multicolumn{1}{c}{$\ldots$} & \multicolumn{1}{c}{$z_{1k_2}$}   \\
		Cluster $\U_2$     &
		\multicolumn{1}{c}{$z_{21}$} & \multicolumn{1}{c}{$z_{22}$} & \multicolumn{1}{c}{$\ldots$} & \multicolumn{1}{c}{$z_{2k_2}$}      \\
		\multicolumn{1}{c}{$\vdots$}     &
		\multicolumn{1}{c}{$\vdots$} & \multicolumn{1}{c}{$\vdots$} & \multicolumn{1}{c}{$\ddots$} & \multicolumn{1}{c}{$\vdots$}\\
		Cluster $\U_{k_1}$     &
		\multicolumn{1}{c}{$z_{k_1 1}$} & \multicolumn{1}{c}{$z_{k_1 2}$} & \multicolumn{1}{c}{$\ldots$} & \multicolumn{1}{c}{$z_{k_1 k_2}$} \\
		\bottomrule
	\end{tabular}
\end{table}

\begin{figure}[t]
	\begin{subfigure}[t]{0.24\textwidth}
		\includegraphics[width=\textwidth]{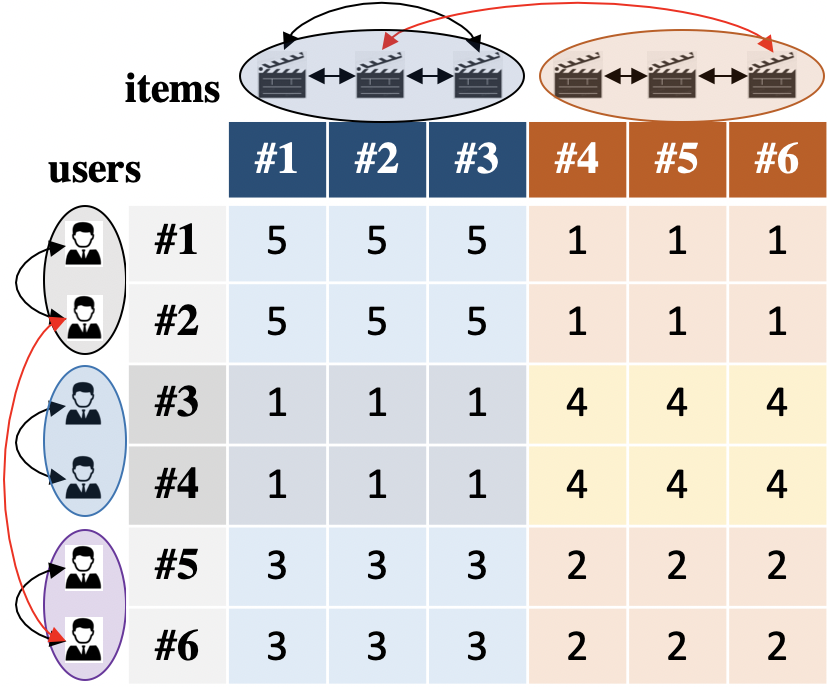}
		\caption{Nominal matrix $\N$}
		\label{fig:1}
	\end{subfigure}
	\begin{subfigure}[t]{0.24\textwidth}
		\includegraphics[width=0.88\textwidth]{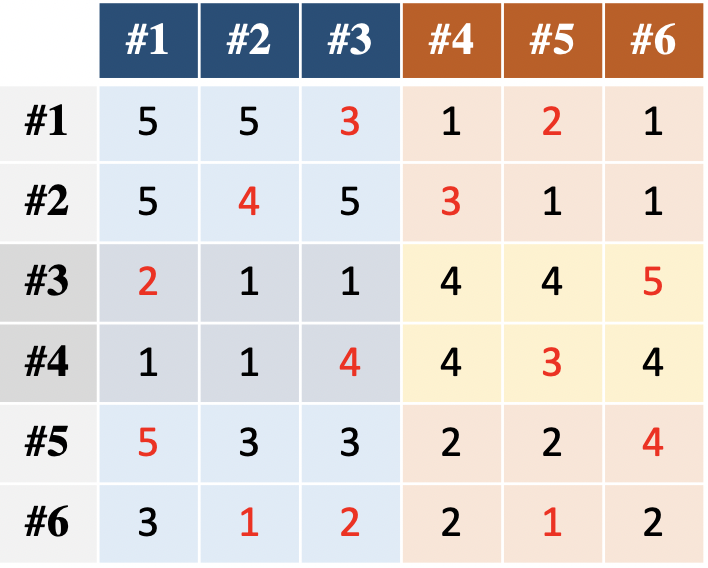}
		\caption{Personalized rating matrix}
		\label{fig:2}
	\end{subfigure}
	\caption{An example with $6$ users (partitioned into $3$ clusters) and $6$ items (partitioned into $2$ clusters). The nominal ratings are chosen from $\mathcal{Z} \in \{1,2,3,4,5\}$, and are set to be $z_{11} = 5, z_{12} = 1, z_{21} = 1, z_{22} = 4, z_{31} = 3, z_{32} = 2$.}
	\label{fig:12}
\end{figure}

As users in the same cluster are more likely to share similar preference (which is called \emph{homophily}~\cite{mcpherson2001birds} in the social sciences), we introduce the notion of {\it nominal ratings} to represent the levels of interest from certain user clusters to certain item clusters. Specifically, for all the users in cluster $\U_a$, their nominal ratings to all the items in cluster $\I_b$ (where $a\in[k_1]$, $b\in[k_2]$) are given by $z_{ab} \in \mathcal{Z}$ (as shown in Table~\ref{table:1}). That is, the nominal ratings given by users in the same clusters are the same, and the nominal ratings received by items in the same cluster are also identical. 
Thus, given $\sigma$ (the labels of $n$ users), $\tau$ (the labels of $m$ items), and $\{z_{ab}\}$ (the nominal ratings), the corresponding {\it nominal matrix} $\N \in \mathcal{Z}^{n \times m}$ is an $n \times m$ matrix that contains the nominal ratings from $n$ users to $m$ items, and each entry $\N_{ij}$ (the nominal rating from user $i$ to item $j$) equals $z_{\sigma(i) \tau(j)}$. An example of the nominal matrix is illustrated in Fig.~\ref{fig:1}.

Our model also has the flexibility that the interest of each individual user may differ from the nominal interest of the cluster he/she belongs to. We model this flexibility by assuming that the {\it personalized rating} $V_{ij} \in \mathcal{Z}$ of user $i$ to item $j$ is a stochastic function of the nominal rating $\mathsf{N}_{ij}$. More precisely, we define $Q_{V|Z} \in \mathcal{P}(\mathcal{Z} \times \mathcal{Z})$ as the \emph{personalization distribution} that reflects the diversity of users in the same cluster. For each user $i$, his/her personalized rating $V_{ij} \in \mathcal{Z}$ to item $j$ is distributed according to $Q_{V|Z=\N_{ij}} \in \mathcal{P}(\mathcal{Z})$. A natural assumption we adopt is that $Q_{V|Z=z}(z) > Q_{V|Z=z}(z')$ for all $z'\ne z$; that is, if the nominal rating is $z \in \mathcal{Z}$, then the personalized rating is \emph{most likely} to be $z$. For a specific user cluster $\U_a$ and an item cluster $\I_b$, all the personalized ratings $\{V_{ij}\}_{i \in \U_a, j \in \I_b}$ (corresponding to all the user-item pairs $(i,j)$ such that $i \in \U_a$ and $j \in \I_b$) follow the same distribution $Q_{V|Z=z_{ab}}$. For simplicity, we abbreviate $Q_{V|Z=z_{ab}}$ as $Q_{ab}$. An example of the personalized rating matrix is illustrated in Fig.~\ref{fig:2}.


\subsection{Observations} \label{sec:observation}
The learner observes three pieces of information: 
\begin{enumerate}[wide, labelwidth=!, labelindent=0pt]
	\item A sub-sampled rating matrix $\mathsf{U}$, with each entry $\mathsf{U}_{ij} = V_{ij}$ with probability (w.p.) $p$ and $\mathsf{U}_{ij} = \sf{e}$ (erasure symbol) w.p. $1 - p$. We refer to $p$ as the {\it sample probability} and $mnp$ as the expected number of sampled entries.   
	
	\item A social graph $G_{1} = (\mathcal{V}_1, \mathcal{E}_1)$, where $\mathcal{V}_1$ is the set of $n$ user nodes. Let $\mathsf{B}$ be a $k_1 \times k_1$ symmetric {\it connectivity matrix} that represents the probabilities of connecting two nodes in $G_1$. Each pair of nodes $(i,i')$ is connected (i.e., $(i,i') \in \mathcal{E}_1$) independently w.p. $\mathsf{B}_{\sigma(i) \sigma(i')}$. 
	
	\item An item graph $G_{2} = (\mathcal{V}_2, \mathcal{E}_2)$, where $\mathcal{V}_2$ is the set of $m$ item nodes. Let $\mathsf{B}'$ be a $k_2 \times k_2$ symmetric connectivity matrix that represents the probabilities of connecting two nodes in $G_2$. Each pair of nodes $(j,j')$ is connected (i.e., $(j,j') \in \mathcal{E}_2$) independently w.p. $\mathsf{B}'_{\tau(j)\tau(j')}$. 
\end{enumerate}

\subsection{Objective} The learner is tasked to design an estimator $\phi=\phi(\mathsf{U}, G_1,G_2)$ to exactly recover both the user clusters $\{\U_a\}_{a\in[k_1]}$ and item clusters $\{\I_b\}_{b \in [k_2]}$  (or equivalently, the label functions $\sigma$ and $\tau$), as well as to reconstruct the nominal matrix $\mathsf{N}$. The output of the estimator $\phi$ is denoted by $(\widehat{\sigma}, \widehat{\tau}, \widehat{\mathsf{N}})$.

To measure the accuracies of the estimated label functions $\widehat{\sigma}$ and $\widehat{\tau}$, we define the {\em misclassification proportions} as 
\begin{align}
l_1(\widehat{\sigma}, \sigma) &:= \min_{\pi \in \mathcal{S}_{k_1}} \frac{1}{n} \sum_{i\in[n]} \mathbbm{1}\{\widehat{\sigma}(i) \ne \pi(\sigma(i)) \}, \label{eq:weak1} \\
l_2(\widehat{\tau}, \tau) &:= \min_{\pi \in \mathcal{S}_{k_2}} \frac{1}{m} \sum_{j\in[m]} \mathbbm{1}\{\widehat{\tau}(j) \ne \pi(\tau(j)) \},\label{eq:weak2}
\end{align} 
where $\mathcal{S}_{k_1}$ (resp. $\mathcal{S}_{k_2}$) is the set of all permutations of $[k_1]$ (resp. $[k_2]$). The permutations are introduced because it is only possible to recover the \emph{partitions} of users/items, rather than the actual labels (i.e., the best we can hope for is to ensure $l_1(\widehat{\sigma}, \sigma) = 0$ and $l_2(\widehat{\tau}, \tau) = 0$).


Furthermore, we also define the concept of {\em weak recovery} which plays a role in the intermediate steps of our algorithm. 
\begin{definition} \label{def:weak}
An estimate  $\widehat{\sigma}$ (resp. $\widehat{\tau}$) is said to achieve  \emph{weak recovery} if the misclassification proportion $l_1(\widehat{\sigma}, \sigma) \to 0$ as $n$ tends to infinity (resp. $l_2(\widehat{\tau}, \tau) \to 0$ as $m$ tends to infinity).
\end{definition}

\section{{\sc Mc2g}: A computationally efficient, statistically optimal algorithm} \label{sec:algorithm}

In this section, we present a computationally efficient multi-stage algorithm called {\sc Mc2g} for recovering the clusters of users and items, and the nominal matrix $\mathsf{N}$,  given the social and item similarity graphs. Knowledge of the model parameters (e.g., connectivity matrices $\mathsf{B}$ and $\mathsf{B}'$ and personalization distribution $Q_{V|Z}$) is not needed for {\sc Mc2g} to succeed,  as they will be estimated on-the-fly. 
Roughly speaking, {\sc Mc2g} consists of four stages: Stage 1 achieves \emph{weak recovery} of the user/item clusters; Stage 2 estimates the model parameters $\R$, $\R'$, and $Q_{V|Z}$; and Stages 3 and 4 respectively refine these estimates of users and items via \emph{local refinements} steps. The inputs include the sub-sampled rating matrix $\mathsf{U}$ and two graphs $G_1$ and $G_2$.  
 
Before describing our algorithm {\sc Mc2g} in detail in Subsection~\ref{sec:detail}, we want to first point out a common issue that often arises in the analysis of multi-stage algorithms.
When analyzing the error probability of multi-stage algorithms, one needs to be cognizant of the dependencies between random variables in different stages. For example, a pair of random variables that are initially independent may become dependent conditioned on the success of a preceding stage. We circumvent this issue by using an \emph{information splitting} method inspired by prior works~\cite{chin2015stochastic,abbe2015exact,abbe2015community} on community detection.  As a concrete example, Fig.~\ref{fig:partition} illustrates how we split the information of the social graph into two pieces, where the first piece is for Stage 1 and the second piece is for subsequent stages.  Information splitting can be viewed as a preliminary step for our main algorithm {\sc Mc2g}, and is formally described in Section~\ref{sec:split}.

\begin{remark}{\em
An alternative approach to circumvent the aforementioned issue of dependence is to use the so-called {\em uniform analysis} technique, which has been adopted by some other works~\cite{chen2016community, ahn2018binary, jo2020discrete}. However, this requires more rounds of local refinements and thus increases the computational complexity. }
\end{remark}

\subsection{Information splitting} \label{sec:split}
The high-level idea is to split the observations $(\UU, G_1, G_2)$ into two parts---the first part, denoted as $(G_1^\mathrm{a}, G_2^\mathrm{a})$, is used for weak recovery of users and items in Stage 1; while the second part, denoted as $(\UU, G^\mathrm{b}_1, G^\mathrm{b}_2)$, is used for estimating the parameters  and for local refinements (exact recovery) of each user and item in Stages 2--4. We elaborate on the information splitting method as follows.

\begin{figure}[t]
		\includegraphics[width=0.45\textwidth]{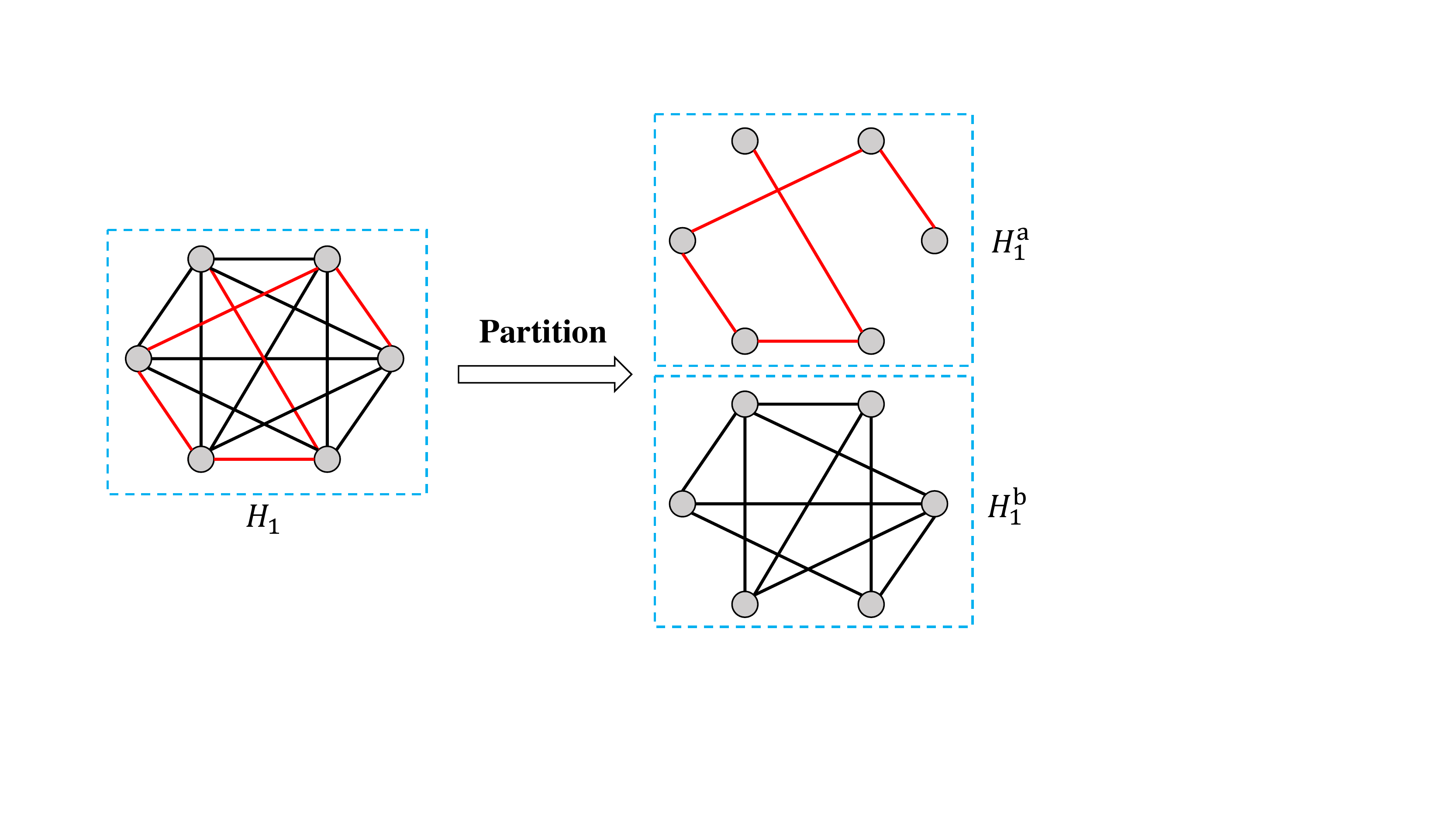}
	\caption{The partition of a complete graph $H_1$ with $n = 6$ nodes  into two sub-graphs $H_1^{\mathrm{a}}$ and $H_1^{\mathrm{b}}$.}
	\label{fig:partition}
\end{figure}

\begin{enumerate}[wide, labelwidth=!, labelindent=0pt]
	\item Let $H_1 = (\mathcal{V}_1, \bE_1)$ be the \emph{complete graph} with vertex set $\mathcal{V}_1 = [n]$ and edge set $\bE_1$ which contains all the  $\binom{|\mathcal{V}_1|}{2}$ edges on $\mathcal{V}_1$. We randomly partition $H_1$ into two sub-graphs $H_1^{\mathrm{a}} = (\mathcal{V}_1, \bE_1^\mathrm{a})$ and $H_1^{\mathrm{b}} = (\mathcal{V}_1, \bE_1^\mathrm{b})$ such that $H_1^{\mathrm{a}}$ is an {\it Erd\H{o}s-R\'{e}nyi (ER) graph} on $\mathcal{V}_1$ with edge probability $1/\sqrt{\log n}$. That is, each $e \in \bE_1$ is sampled (independently) to $\bE_1^\mathrm{a}$ with probability $1/\sqrt{\log n}$, and to $\bE_1^\mathrm{b}$ with probability $1 - 1/\sqrt{\log n}$, where $\bE_1^{\mathrm{b}}$ is the complement of $\bE_1^{\mathrm{a}}$. An example is illustrated in Fig.~\ref{fig:partition}.  This partition is done independently of the generation of the SBM $G_1$. For any realizations $H_1^{\mathrm{a}} = h_1^{\mathrm{a}}$ and $H_1^{\mathrm{b}} = h_1^{\mathrm{b}}$, let 
	\begin{align}
	G_1^\mathrm{a} := h_1^{\mathrm{a}} \cap G_1 \ \text{ and } \ G_1^\mathrm{b} := h_1^{\mathrm{b}} \cap G_1. 
	\end{align}
	be two sub-SBMs on sub-graphs $h_1^{\mathrm{a}}$ and $h_1^{\mathrm{b}}$, respectively.\footnote{With a slight abuse of notations, we use $h_1^{\mathrm{a}} \cap G_1$ (resp. $h_1^{\mathrm{b}} \cap G_1$) to represent a graph with edge set being the intersection between the edge sets of $h_1^{\mathrm{a}}$ (resp, $h_1^{\mathrm{b}}$) and $G_1$. More specifically, for the sub-SBM $G_1^\mathrm{a}$ (resp. $G_1^\mathrm{b}$), any pairs of nodes $(i,i')$ are connected with probability $\R_{\sigma_i \sigma_{i'}}$ if $(i,i') \in \bar{\mathcal{E}}_1^\mathrm{a}$ (resp. $(i,i') \in \bar{\mathcal{E}}_1^\mathrm{b}$), and with probability zero otherwise.}
	
	\item  Similarly, let $H_2 = (\mathcal{V}_2, \bE_2)$ be the complete graph with vertex set $\mathcal{V}_2 = [m]$ and edge set $\bE_2$, $H_2^{\mathrm{a}}$ is an ER graph on $\mathcal{V}_2$ with edge probability $1/\sqrt{\log m}$, and $\bE_2^{\mathrm{b}}$ is the complement of $\bE_2^{\mathrm{a}}$. For any $H_2^{\mathrm{a}} = h_2^{\mathrm{a}}$ and $H_2^\mathrm{b} = h_2^\mathrm{b}$, we also define 
	\begin{align}
	G_2^\mathrm{a} := h_2^{\mathrm{a}} \cap G_2 \ \text{ and } \ G_2^\mathrm{b} := h_2^\mathrm{b} \cap G_2. 
	\end{align}
\end{enumerate}

\begin{algorithm}
	\SetAlgoLined
	\SetKwInOut{Input}{Input}\SetKwInOut{Output}{Output}
	\SetKwData{Left}{left}\SetKwData{This}{this}\SetKwData{Up}{up}
	\Input{$(\UU, G_1, G_2) = (G_1^\mathrm{a}, G_2^\mathrm{a}) \cup (\UU, G_1^\mathrm{b}, G_2^\mathrm{b})$}
	\Output{Clusters $\{\widehat{\U}_a \}_{a\in[k_1]}$ and $\{\widehat{\I}_b \}_{b\in[k_2]}$ (or label functions $\widehat{\sigma}$ and $\widehat{\tau}$), nominal matrix $\widehat{\mathsf{N}}$}
	
	\BlankLine
	\textbf{Stage 1 (Weak recovery of communities)} \\
	Apply the spectral clustering method on $G_1^\mathrm{a}$ and $G_2^\mathrm{a}$ to obtain initial estimates $\{\U_a^{(0)}\}_{a\in[k_1]}$ and $\{\I_b^{(0)}\}_{b\in[k_2]}$\; 
	
	\BlankLine
	\textbf{Stage 2 (Parameters estimation)} \\
	Estimate connectivity matrices $\mathsf{B}$ and  $\mathsf{B}'$ as per~\eqref{eq:est1}-\eqref{eq:est2};\\
	Estimate personalization distribution $\{\widehat{Q}_{ab}\}$ as per~\eqref{eq:eq}; 
	
	\BlankLine
	\textbf{Stage 3 (Local refinements of users)} \\
	\For{user $i = 1$ $\KwTo$ $n$}{
	Calculate likelihood functions $\{L_a(i)\}_{a\in[k_1]}$;\\
	Let $a_i^* = \argmax_{a \in [k_1]}L_a(i)$, and declare $i \in \U_{a_i^*}$; \\	
	}
	
	\BlankLine
	\textbf{Stage 4 (Local refinements of items)} \\
	\For{item $j = 1$ $\KwTo$ $m$}{
	Calculate likelihood functions $\{L'_b(j)\}_{b\in[k_2]}$;\\
	Let $b_j^* = \argmax_{b \in [k_2]}L'_b(j)$, and declare $j \in \I_{b_j^*}$; \\
	}
	Reconstruct the nominal matrix $\widehat{\mathsf{N}}$ as Per~\eqref{eq:N2}.  
	\caption{{\sc Mc2g}}
	\label{algorithm:1}
\end{algorithm}

\subsection{Algorithm description} \label{sec:detail} 

{\it Stage 1 (Weak recovery of clusters):} We run a spectral clustering method\footnote{To achieve weak recovery of the clusterings of users and items, one can also apply different variants of spectral clustering methods~\cite{abbe2015community,chin2015stochastic,lei2015consistency,gao2017achieving}, semidefinite programming-based methods~\cite{javanmard2016phase}, belief propagation-based methods~\cite{mossel2016density}, or non-backtracking matrix-based methods~\cite{krzakala2013spectral}.} (e.g., Agorithm 2 in~\cite{NIPS2016_a8849b05}) on the social graph $G_1^\mathrm{a}$ to obtain an initial estimate of the label function $\sigma$ (denoted by $\sigma^{(0)}$), and also run a spectral clustering method on the item graph $G_2^\mathrm{a}$  to obtain an initial estimate of the label function $\tau$ (denoted by $\tau^{(0)}$). The estimated user clusters corresponding to $\sigma^{(0)}$ are denoted by $\{\U_a^{(0)}\}_{a\in[k_1]}$ (i.e., $\mathcal{U}_a^{(0)}=(\sigma^{(0)})^{-1}(a)$), and the estimated item clusters corresponding to $\tau^{(0)}$ are denoted by $\{\I_b^{(0)}\}_{b\in[k_2]}$.   These initial estimates $\sigma^{(0)}$ and $\tau^{(0)}$ are expected to serve as good approximations of the true clusters, such that both $\sigma^{(0)}$ and $\tau^{(0)}$ satisfy the weak recovery criteria defined in Definition~\ref{def:weak}. 

{\it Stage 2 (Parameters estimation):} 
For any two sets of nodes $\mathcal{V}$ and $\mathcal{V}'$, the number of edges connecting $\mathcal{V}$ and $\mathcal{V}'$ (in $G_1^\mathrm{b}$ or $G_2^\mathrm{b}$)  is denoted as $e(\mathcal{V}, \mathcal{V}')$. 
Based on the initial estimates $\{\U_a^{(0)}\}_{a\in[k_1]}$ and $\{\I_b^{(0)}\}_{b\in[k_2]}$, we then obtain the MLE for the connectivity matrices $\mathsf{B}$ and $\mathsf{B}'$ of the social and item graphs as 
\begin{align}
&\widehat{\mathsf{B}}_{aa'} = \begin{cases} e(\U_a^{(0)},\U_{a'}^{(0)})/\binom{|\U_a^{(0)}|}{2}, &\text{if } a = a';  \\
e(\U_a^{(0)},\U_{a'}^{(0)})/(|\U_a^{(0)}|\cdot|\U_{a'}^{(0)}|), &\text{if } a \ne a'; 
\end{cases}
\label{eq:est1} \\
&\widehat{\mathsf{B}}'_{bb'} = \begin{cases} e(\I_b^{(0)},\I_{b'}^{(0)})/\binom{|\I_b^{(0)}|}{2}, \  \ &\text{if } b = b';  \\
e(\I_b^{(0)},\I_{b'}^{(0)})/(|\I_b^{(0)}|\cdot |\I_{b'}^{(0)}|), \ \  &\text{if } b \ne b'; 
\end{cases}
\label{eq:est2}
\end{align}
where $a,a' \in [k_1]$ and $b,b' \in [k_2]$. Moreover, we define sets of $(i,j)$-pairs $\mathcal{Q}_{ab}^z$ (where $a \in [k_1]$, $b\in[k_2]$, and $z\in \mathcal{Z}$) as  $$\mathcal{Q}_{ab}^z := \left\{(i,j): \mathsf{U}_{ij} = z, i \in \U_a^{(0)}, j \in \I_b^{(0)}  \right\}.$$
Then, the estimated personalization distribution is given by 
\begin{align}
\widehat{Q}_{ab}(z) := \frac{|\mathcal{Q}_{ab}^z|}{\sum_{z \in \mathcal{Z}} |\mathcal{Q}_{ab}^z|}, \quad \forall a \in [k_1], b \in [k_2]. \label{eq:eq}
\end{align}

{\it Stage 3 (Local refinements of users):} 
This stage refines the classification of each user locally, based on the ratings in $\mathsf{U}$, the social graph $G_1^\mathrm{b}$, and the initial estimates $\{\U_a^{(0)}\}_{a\in[k_1]}$ and $\{\I_b^{(0)}\}_{b\in[k_2]}$. For each user $i \in [n]$, we essentially adopt a \emph{local MLE} to determine which cluster it belongs to. We define the \emph{likelihood function} that reflects how likely user $i$ belongs to cluster $\U_a$ as:
\begin{align} 
L_a(i) &:= \sum_{a'\in[k_1]} e(\{i\}, \U^{(0)}_{a'})\cdot \log\left(\widehat{\R}_{aa'}/(1-\widehat{\R}_{aa'})\right) \notag \\
&\qquad\qquad + \sum_{b \in [k_2]}\sum_{j \in \I_b^{(0)}} \mathbbm{1}\{\UU_{ij} \!\ne\! \mathsf{e} \} \cdot \log \widehat{Q}_{ab}(\UU_{ij}). \label{eq:La}
\end{align}
Let $a_i^\ast := \argmax_{a \in [k_1]}L_a(i)$ be the index of the most likely user cluster for user $i$. {\sc Mc2g} then declares $i \in \widehat{\U}_{a_i^*}$; or equivalently, $\widehat{\sigma}(i) = a_i^*$.

{\it Stage 4 (Local refinements of items):} This stage refines the classification of each item locally, based on $\mathsf{U}$, $G_2^\mathrm{b}$, and the initial estimates $\{\U_a^{(0)}\}_{a\in[k_1]}$ and $\{\I_b^{(0)}\}_{b\in[k_2]}$. For each item $j \in [m]$, we define the likelihood function that reflects how likely item $j$ belongs to cluster $\I_b$ as:  
\begin{align}
L'_b(j) &:= \sum_{b'\in[k_2]} e(\{j\}, \I^{(0)}_{b'})\cdot \log\left(\widehat{\R}'_{bb'}/(1-\widehat{\R}'_{bb'})\right) \notag \\
&\qquad\quad \ + \sum_{a \in [k_1]}\sum_{i \in \U_a^{(0)}} \mathbbm{1}\{\UU_{ij} \!\ne\! \mathsf{e} \} \cdot \log \widehat{Q}_{ab}(\UU_{ij}).
\end{align}
Let $b_j^\ast := \argmax_{b \in [k_2]}L'_b(j)$ be the index of the most likely item cluster for item $j$. {\sc Mc2g} then declares $j \in \widehat{\I}_{b_j^*}$; or equivalently, $\widehat{\tau}(j) = b_j^*$.

Finally, one can recover the nominal matrix $\widehat{\mathsf{N}}$ by setting
\begin{align}
\widehat{\mathsf{N}}_{ij}= \argmax_{z \in \mathcal{Z}} \ \widehat{Q}_{ab}(z), \quad \mathrm{for} \ i \in \widehat{\U}_a, j \in \widehat{\I}_b. \label{eq:N2}
\end{align}

\begin{remark} \label{remark:simplified}
\em
The information splitting method introduced in this section is merely for the purpose of analysis (as discussed in the second paragraph of this section); however, it may not be practical when $n$ and $m$ are \emph{not sufficiently large}, in which case the first part of the graphs $(G_1^\mathrm{a}, G_2^\mathrm{a})$ may be too sparse to achieve weak recovery of the true clusters in Stage 1. Thus, in practice, instead of splitting the graphs $(G_1, G_2)$ into $(G_1^\mathrm{a}, G_2^\mathrm{a})$ (on which Stage 1 is applied) and $(G_1^\mathrm{b}, G_2^\mathrm{b})$ (on which Stages 2--4 are applied), one can skip the information splitting step in Section~\ref{sec:split} and simply apply every stage on the fully-observed graphs $(G_1, G_2)$ for weak recovery, parameter estimations, and local refinements---this is referred to as the \emph{simplified version of {\sc Mc2g}}. In our experiments (Section~\ref{sec:experiment}), we adopt this simplified version of {\sc Mc2g}, and show that it also works well on both synthetic and real datasets. 
\end{remark}


\subsection{Computational Complexity}
Using the iterative power method~\cite{halko2011finding}, the spectral clustering method used to obtain initial estimates of $G^\mathrm{a}_1$ and $G^\mathrm{a}_2$ run in times at most $\mathcal{O}(|\mathcal{E}_1|\log n)$ and $\mathcal{O}(|\mathcal{E}_2|\log m)$ respectively, where $|\mathcal{E}_1| = \mathcal{O}(n \log n)$ and $|\mathcal{E}_2| = \mathcal{O}(m \log m)$ with high probability. In each of the following steps, {\sc Mc2g} requires (at most) a {\em single} pass of all the sub-sampled entries in the rating matrix $\mathsf{U}$ and the edge sets $\mathcal{E}_1$ and $\mathcal{E}_2$, which amounts to at most $\mathcal{O}(\max\{n \log n, m \log m\})$ time. Therefore, the overall computational complexity is $\mathcal{O}(\max\{n (\log n)^2, m (\log m)^2\})$ (i.e., quasilinear in $m$ and $n$) with high probability.

\section{Theoretical guarantees of Mc2g and Information-theoretic lower bounds} \label{sec:result}
This section provides  theoretical guarantees  for {\sc Mc2g}. Under the symmetric setting defined in Subsection~\ref{sec:symmetric}, we characterize the expected number of sampled entries required for {\sc Mc2g} to succeed; the key message there is that this quantity depends critically on (i) the ``qualities'' of the social and item similarity graphs, and (ii) the \emph{squared Hellinger distance} between the rating statistics of different user/item clusters. We further establish an information-theoretic lower bound on the expected number of sampled entries. This bound matches the achievability bound  up to a constant factor, thus demonstrates the order-wise optimality of {\sc Mc2g}.

\subsection{The Symmetric Setting} \label{sec:symmetric}
Under the symmetric setting, it is assumed that (i) the user clusters are of equal size (i.e., $|\U_a| = n/k_1$ for all $a \in [k_1]$) and the item clusters are of equal size (i.e., $|\I_b| = m/k_2$ for all $b \in [k_2]$)\footnote{We implicitly assume that $n$ is divisible by $k_1$ and $m$ is divisible by $k_2$.  In the case that $n$ and $m$ are not multiples of $k_1$ and $k_2$ respectively, rounding operations required to define the set $\Xi$. Such rounding operations, however, do not affect the calculations and results downstream. }, and (ii) the connection probability for each pair of nodes depends only on whether they belong to the same cluster, i.e., the connectivity matrices $\R$ and $\R'$ satisfy
\begin{align}
\R_{aa'} = \begin{cases} \alpha_1, & \text{if } a = a'; \\
\beta_1, & \text{if } a \ne a'; \end{cases} \ \text{and} \ \ \R'_{bb'} = \begin{cases} \alpha_2, & \text{if } b = b'; \\
\beta_2, & \text{if } b \ne b'. \end{cases} \notag 
\end{align}
Similar to the prior work~\cite{eric2021}, we assume $m = \omega(\log n)$ and $n = \omega(\log m)$ such that $m \to \infty$ as $n\to \infty$. 

We note that {\sc Mc2g} is not restricted to the symmetric setting; it can be applied more generally  to asymmetric scenarios. Indeed, for the experiments in Section~\ref{sec:experiment}, we do not make the symmetric assumption. In this section, however, we make this assumption to simplify the presentation of Theorem~\ref{thm:algorithm} and to clearly understand the effect of the parameters of the model on the minimum expected number of sampled entries required for {\sc Mc2g} to succeed.

In the following, we formally define the notion of \emph{exact recovery}. Note that the model is governed by the pair of label functions $(\sigma, \tau)$ together with the nominal matrix $\N$, and we define the \emph{parameter space} that contains all valid $(\sigma, \tau,\N)$ under the symmetric setting as
\begin{align*}
\Xi \triangleq &\Big\{\!(\sigma, \tau,\N)\big| \ \sigma\!:\! [n] \!\to\! [k_1],  \ \big|\{i\!\in\![n]\!:\sigma_i\!=\!a\}\big|\!=\!\frac{n}{k_1}, \forall a\!\in\![k_1]; \\ 
&\ \ \tau\!:\! [m] \to [k_2], \  \big|\{j\in[m]:\tau_j=b\}\big|\!=\!\frac{m}{k_2}, \ \forall b \in [k_2];\\ 
&\ \ \N \in \mathcal{Z}^{n \times m}, \N_{ij} \!=\! \N_{i'j'} \text{ if } \sigma(i)\!=\!\sigma(i') \text{ and } \tau(j) \!=\! \tau(j') \!\Big\}.
\end{align*} 
Let $(\sigma,\tau,\N)$ be the ground truth, and $(\widehat{\sigma},\widehat{\tau}, \widehat{\N})$ be the output of the estimator $\phi=\phi(\UU, G_1,G_2)$. We say the event  $\mathcal{E}_{(\sigma,\tau,\N)}$ occurs if the output $(\widehat{\sigma},\widehat{\tau}, \widehat{\N})$ of the estimator $\phi$ satisfies one of the following three criterions: (i) $\{l_1(\widehat{\sigma}, \sigma) \ne 0\}$, (ii) $\{l_2(\widehat{\tau}, \tau) \ne 0\}$, and (iii) $\{\widehat{\N} \ne \N \}$.

\begin{definition}[Exact recovery]{\em
For any estimator $\phi$, its corresponding \emph{(maximum) error probability} is defined as $$P_{\text{err}}(\phi) := \max_{(\sigma,\tau,\N) \in \Xi}\PP_{(\sigma,\tau,\N)}\big(\phi(\mathsf{U}, G_1,G_2) \in \mathcal{E}_{(\sigma,\tau,\N)}\big),$$ 
where $\PP_{(\sigma,\tau,\N)}(\cdot)$ is the probability when $(\mathsf{U}, G_1,G_2)$ is generated according to the model governed by $(\sigma,\tau,\N)$. A sequence of estimators $\Phi = \{\phi_n\}_{n=1}^{\infty}$ achieves \emph{exact recovery} if
\begin{align}
\lim_{n \to \infty} P_{\text{err}}(\phi_n) = 0. \label{eq:2}
\end{align}
}
\end{definition}

\begin{definition}[Sample complexity] \label{def:sample}{\em
	The {\em sample complexity} is defined as the minimum expected number of samples in the matrix $\mathsf{U}$ such that there exists $\Phi$ for which~\eqref{eq:2} holds.}
\end{definition}


\subsection{Theoretical 	 guarantees of {\sc Mc2g}}
As we shall see, the ``qualities'' of the social and item graphs play a key role in the performance of {\sc Mc2g}. 
Specifically, we define a measure of the quality of the social graph $G_1$ as $I_1 := n(\sqrt{\alpha_1} - \sqrt{\beta_1})^2/(\log n)$. A larger value of $I_1$ implies a better quality of the graph, since the structures of the clusters  are more clearly delineated when the difference between the intra-cluster probability $\alpha_1$ and the inter-cluster probability $\beta_1$ is larger. 
Analogously, we define a measure of the quality of the item graph $G_2$ as  $I_2 := m(\sqrt{\alpha_2} - \sqrt{\beta_2})^2/(\log m)$. 

The performance of {\sc Mc2g} also depends on the statistics of the rating matrix. Intuitively, if the rating statistics of two clusters are further apart, it is then easier to distinguish them. It turns out that under the symmetric setting, the distance between the rating statistics of different clusters can be measured by the \emph{squared Hellinger distance}:
$$H^2(P,Q) := 1 - \sum_{z \in \mathcal{Z}}\sqrt{P(z)Q(z)},$$
for probability distributions $P$ and $Q$. 
We then define $d(\U_a, \U_{a'}) := \sum_{b \in [k_2]} H^2(Q_{ab},Q_{a'b})$ as a measure of the discrepancy between user clusters $\U_a$ and $\U_{a'}$ (where $a,a'\in[k_1]$), and $d_{\U} := \min_{a \ne a'} d(\U_a, \U_{a'})$ as the minimal discrepancy over all pairs of user clusters. A larger value of $d_{\U}$ means that it is easier to distinguish all the user clusters. Analogously, we define the discrepancy between item clusters $\I_b$ and $\I_{b'}$ (where $b,b' \in [k_2]$) as $d(\I_b, \I_{b'}) := \sum_{a \in [k_1]} H^2(Q_{ab},Q_{ab'})$, and $d_{\I} := \min_{b \ne b'} d(\I_b, \I_{b'})$ as the minimal discrepancy over all pairs of item clusters.

\begin{remark}{\em 
The squared Hellinger distance $H^2(P,Q)$ satisfies $H^2(P,Q)\in [0,1]$ and $H^2(P,Q)=0$ if and only if $P=Q$.}
\end{remark}

Theorem~\ref{thm:algorithm} below states the expected number of sampled entries needed for {\sc Mc2g} to succeed under the symmetric setting.

\begin{theorem}[Performance of {\sc Mc2g}] \label{thm:algorithm}
For any $\epsilon > 0$, if the expected number of sampled entries $mnp$ satisfies 
\begin{align}
mnp \ge &\max\!\Bigg\{\frac{\!\left[(1\!+\!\epsilon) \!-\! \frac{I_{1}}{k_1}\right] n\log n}{d_{\U}/k_2},\!  \frac{\left[(1\!+\!\epsilon) \!-\! \frac{I_{2}}{k_2}\right]m\log m}{d_{\I}/k_1} \Bigg\}, \label{eq:achievability1}
\end{align}
then {\sc Mc2g} ensures $P_{\emph{err}} \to 0$ as  $n \to \infty$.
\end{theorem}

\begin{remark} \label{remark:thm}{\em
Some remarks on Theorem~\ref{thm:algorithm} are in order.
\begin{enumerate}[wide, labelwidth=!, labelindent=0pt]
\item Roughly speaking, the first term on the RHS of~\eqref{eq:achievability1} is the threshold for Stage 3 (local refinements of users) to succeed. This is because when $mnp$ exceeds the first term, the probability that a single user is misclassified to an incorrect cluster (in Stage 3) is at most $n^{-\ell}$ for some $\ell > 1$. Thus, taking a union bound over all the $n$ users still results in a vanishing error probability. Similarly, the second term on the RHS of~\eqref{eq:achievability1} is the threshold for Stage 4 (local refinements of items) to succeed.
  
\item Our result in~\eqref{eq:achievability1} confirms our intuitive belief that increasing $d_{\U}$ and $d_{\I}$ (the minimum discrepancies between user and item clusters) indeed helps to reduce the number of samples required for exact recovery. Similarly, increasing $I_1$ and $I_2$ (the qualities of the social and item graphs) also helps to reduce the sample complexity. 
\item It is also worth noting that when $I_1 > k_1$, the first term in~\eqref{eq:achievability1} becomes non-positive (thus inactive); this means that performing local refinements of users in Stage 3 is no longer needed, which is due to the fact that the spectral clustering method in Stage 1 has already ensured exact recovery of the $k_1$ user clusters.  This observation coincides with the theoretical result of community detection in the symmetric SBM~\cite{abbe2015community}, which states that exact recovery of $k_1$ clusters is possible when $I_1 > k_1$. Similarly, when $I_2 > k_2$, performing local refinements of items in Stage 4 is no longer needed, as the spectral clustering method in Stage 1 has already ensured exact recovery of the $k_2$ item clusters.      

\item While the theoretical result in Theorem~\ref{thm:algorithm} is dedicated to this symmetric setting, {\sc Mc2g} is applicable to a more general matrix completion problem with social and item similarity graphs, where the sizes of user/item clusters may be different.  This is confirmed by the experiments in Section~\ref{sec:experiment}. 
\end{enumerate}}
\end{remark}

\subsection{Information-theoretic lower bound}

Theorem~\ref{thm:2} below provides an information-theoretic lower bound on the sample complexity under the symmetric setting. Again, the lower bound is a function of $I_1$, $I_2$ (the quality of the social/item graph), and $d_{\U}$ and $d_{\I}$ (the minimum discrepancies measured in terms of the squared Hellinger distances of user/item clusters).

\begin{theorem}[Impossibility result]\label{thm:2}
	For any $\epsilon > 0$,~if
	\begin{align}
	mnp \!<\! \max\!\Bigg\{\! \frac{\!\left[\frac{1-\epsilon}{2} \!-\! \frac{I_{1}}{k_1}\right]\! n\log n}{d_{\U}/k_2},\!  \frac{\left[\frac{1-\epsilon}{2} \!-\! \frac{I_{2}}{k_2}\right]\!m\log m}{d_{\I}/k_1} \!\Bigg\}, \label{eq:converse1}
	\end{align}
	then $\lim_{n \to \infty} P_{\emph{err}}(\phi) = 1$ for any estimator $\phi$.
\end{theorem}

Theorem~\ref{thm:2} states that {\em any} estimator {\it must} necessarily fail if the expected number of samples is smaller than the maximal term in~\eqref{eq:converse1}. Thus, the sample complexity defined in Definition~\ref{def:sample} is upper-bounded by the RHS of~\eqref{eq:achievability1}, and lower-bounded by the RHS of~\eqref{eq:converse1}. In particular, Theorem~\ref{thm:2} guarantees that $P_{\text{err}}$ approaches one as $n\to\infty$; this is the so-called \emph{strong converse}~\cite{wolfowitz2012coding} in the information theory parlance. Comparing~\eqref{eq:converse1} with the achievability bound in~\eqref{eq:achievability1}, we note that they match up to a constant factor, and this further demonstrates the order-wise optimality of the proposed computationally efficient algorithm {\sc Mc2g}.

\section{Proof of Theorem~\ref{thm:algorithm}} \label{sec:proof}

\noindent{\textbf{Analysis of Stage 1:}}  Note that the sub-SBM $G_1^\mathrm{a}$ is generated on the sub-graph $h_1^{\mathrm{a}}$; thus the performance of the spectral clustering method on $G_1^\mathrm{a}$ essentially depends on the realization $h_1^{\mathrm{a}}$. A similar argument also applies to $G_2^\mathrm{a}$.

To circumvent the difficulties of analyzing fixed $h_1^{\mathrm{a}}$ and $h_1^{\mathrm{b}}$, we first consider two artificial SBMs $\widetilde{G}_1$ and $\widetilde{G}_2$, where $\widetilde{G}_1$ is generated on the $n$ user nodes and has connectivity matrix $\R/\sqrt{\log n}$, and $\widetilde{G}_2$ is generated on the $m$ item nodes and has connectivity matrix $\R'/\sqrt{\log m}$. 
A prior result in~\cite[Theorem 6]{NIPS2016_a8849b05} shows that there exist vanishing sequences $\epsilon_n$, $\eta_n$, and $\gamma_n$ (depending on $\R$ and $\R'$) such that with probability at least $1 - \epsilon_n$, the spectral clustering method running on $\widetilde{G}_1$ and $\widetilde{G}_2$ respectively ensure that
\begin{align}
l_1(\sigma^{(0)}, \sigma) \le \eta_n \ \text{ and } \ l_2(\tau^{(0)}, \tau) \le \gamma_n. \label{eq:s1}
\end{align}
Based on the good performances of spectral clustering methods running on $\widetilde{G}_1$ and $\widetilde{G}_2$, we next show that spectral clustering methods running on $G_1^\mathrm{a}$ and $G_2^\mathrm{a}$ also provide satisfactory initialization results with high probability.

\begin{definition} \label{def:2}
{\em Let $\mathbf{h} = (h_1^{\mathrm{a}}, h_1^{\mathrm{b}}, h_2^{\mathrm{a}}, h_2^{\mathrm{b}})$ be an aggregation of realizations of the sub-graphs.
	\begin{enumerate}[wide, labelwidth=!, labelindent=0pt]
	\item A sub-graph $h_1^{\mathrm{a}}$ is said to be \emph{good} if the probability that ``a spectral clustering method running on $G_1^\mathrm{a}$ (which depends on $h_1^{\mathrm{a}}$) ensures $l_1(\sigma^{(0)}, \sigma) \le \eta_n$'' is at least $1 - \sqrt{\epsilon_n}$.  A sub-graph $h_1^{\mathrm{b}}$ is said to be \emph{good} if the degree of any node in $h_1^{\mathrm{b}}$ is at least $n(1-2/\sqrt{\log n})$.
	
	\item A sub-graph $h_2^{\mathrm{a}}$ is said to be \emph{good} if the probability that ``a spectral clustering method running on $G_2^\mathrm{a}$ ensures $l_2(\tau^{(0)},\tau) \le \gamma_n$'' is at least $1 - \sqrt{\epsilon_n}$. A sub-graph $h_2^\mathrm{b}$ is said to be \emph{good} if the degree of any node in $h_2^\mathrm{b}$ is at least $m(1-2/\sqrt{\log m})$. 
	
	\item Let $\mathcal{G}$ and $\mathcal{B}$ be two disjoint sets of $\mathbf{h}$. We say $\mathbf{h} \in \mathcal{G}$ if all the elements in $\mathbf{h}$ are good, and  $\mathbf{h} \in \mathcal{B}$ otherwise.
	\end{enumerate}	
}
\end{definition}

\begin{lemma} \label{lemma:new}
	The randomly generated sub-graphs $H_1^{\mathrm{a}}, H_1^{\mathrm{b}},H_2^{\mathrm{a}}, H_2^\mathrm{b}$ are all good with probability at least $(1 - \sqrt{\epsilon_n})^2$. Equivalently, we have
	\begin{align}
	\sum_{\mathbf{h} \in \mathcal{G}} \PP(\mathbf{h}) \ge (1 - \sqrt{\epsilon_n})^2. \label{eq:71}
	\end{align}  
\end{lemma}
\begin{proof}
    See Appendix~\ref{appendix:1}.
\end{proof}

We define $\mathcal{G}'$ as the set of label functions that are close to the true label functions $(\sigma,\tau)$, i.e.,
\begin{align}
\mathcal{G}':=\{(\sigma',\tau'): l_1(\sigma',\sigma) \le \eta_n, \ l_2(\tau',\tau) \le \gamma_n \}, \label{eq:g}
\end{align} 
and $\mathcal{B}':=\{(\sigma',\sigma): (\sigma',\sigma) \notin \mathcal{G}' \}$ as the complement of $\mathcal{G'}$. By definition, we know that when the randomly generated sub-graphs $\mathbf{h} \in \mathcal{G}$, running spectral clustering methods on $G_1^\mathrm{a}$ and $G_2^\mathrm{a}$ yields $(\sigma^{(0)}, \tau^{(0)}) \in \mathcal{G}'$ with high probability, i.e.,
\begin{align}
\sum_{(\sigma^{(0)}, \tau^{(0)}) \in \mathcal{G}'} \PP((\sigma^{(0)}, \tau^{(0)})|\mathbf{h}) \ge (1 - \sqrt{\epsilon_n})^2, \label{eq:72}
\end{align}
which is uniform in $\mathbf{h} \in \mathcal{G}$ (i.e., the sequence $\{\epsilon_n\}$ does not depend on $\mathbf{h}$).

\begin{remark}
	Lemma~\ref{lemma:new} above conveys two important messages: (i) Although the sub-graphs $H_1^{\mathrm{a}}$ and $H_2^{\mathrm{a}}$ are much sparser compared to $H_1$ and $H_2$ (or equivalently, the information contained in $H_1^{\mathrm{a}}$ and $H_2^{\mathrm{a}}$ is much less), they still guarantee the success of running spectral clustering methods (with high probability). (ii) The densities of sub-graphs $H_1^{\mathrm{b}}$ and $H_2^\mathrm{b}$ are almost the same as those of $H_1$ and $H_2$, and this property is critical in Stages 2--4 for proving the theoretical guarantees of {\sc Mc2g}.  
\end{remark}

\noindent{\textbf{Analysis of Stage 2:}} 
Note that the estimates  $\widehat{\mathsf{B}}, \widehat{\mathsf{B}}', \{\widehat{Q}_{ab}\}$ in~\eqref{eq:est1}-\eqref{eq:eq} depend on both $\mathbf{h}$ and $(\sigma^{(0)}, \tau^{(0)})$. In Stage 2, we show in Lemmas~\ref{lemma:alpha} and~\ref{lemma:personalization} below that conditioned on $\mathbf{h} \in \mathcal{G}$ and $(\sigma^{(0)}, \tau^{(0)}) \in \mathcal{G}'$, the estimates are accurate with high probability. 

\begin{lemma} \label{lemma:alpha}
	Suppose $\mathbf{h} \in \mathcal{G}$ and $(\sigma^{(0)}, \tau^{(0)}) \in \mathcal{G}'$.	With probability $1 - o(1)$, there exists a sequence $\varepsilon_n  \in \Omega(\max\{\gamma_n,\eta_n,1/\sqrt{\log n} \}) \cap o(1)$ such that for all $a,a' \in [k_1]$ and $b,b' \in [k_2]$, 
	\begin{align}
	&\big|(\widehat{\mathsf{B}}_{aa'} \!-\! \mathsf{B}_{aa'})/\mathsf{B}_{aa'} \big| \le \varepsilon_n, \ \ \big|(\widehat{\R}'_{bb'} \!-\! \R'_{bb'})/\R'_{bb} \big| \le \varepsilon_n. \label{eq:alpha2}
	\end{align}
\end{lemma}
\begin{proof}
	See Appendix~\ref{appendix:alpha}.
\end{proof}

\begin{lemma} \label{lemma:personalization}
	Suppose $\mathbf{h} \in \mathcal{G}$ and $(\sigma^{(0)}, \tau^{(0)}) \in \mathcal{G}'$. With probability $1 - o(1)$, there exists a sequence $\varepsilon'_n  \in \Omega(\max\{\gamma_n,\eta_n,1/\sqrt{\log n} \}) \cap o(1)$ such that for all $a\in[k_1]$, $b\in[k_2]$, and $z \in \mathcal{Z}$, 
	\begin{align}
	\left|\left(\widehat{Q}_{ab}(z)/Q_{ab}(z)\right) - 1 \right| \le \varepsilon'_n. \label{eq:18}
	\end{align}
\end{lemma}
\begin{proof}
	See Appendix~\ref{appendix:personalization}.
\end{proof}

\begin{remark}
	In Lemmas~\ref{lemma:alpha} and~\ref{lemma:personalization} above, we implicitly assume (without loss of generality) that the permutations minimizing $l_1(\sigma^{(0)},\sigma)$ and $l_2(\tau^{(0)},\tau)$ are both the {\em identity permutation}, i.e., $l_1(\sigma^{(0)},\sigma) =  \sum_{i\in[n]} \mathbbm{1}\{\sigma^{(0)}(i) \ne \sigma(i) \}/n$ and $l_2(\tau^{(0)},\tau) =  \sum_{j\in[m]} \mathbbm{1}\{\tau^{(0)}(j) \ne \tau(j) \}/m$ as Per Eqns.~\eqref{eq:weak1} and~\eqref{eq:weak2}. Without this assumption, one needs to introduce the permutations $\pi^*_1$ and $\pi^*_2$ that respectively minimize $l_1(\sigma^{(0)},\sigma)$ and $l_2(\tau^{(0)},\tau)$---this unnecessarily complicates the presentations of Lemmas~\ref{lemma:alpha} and~\ref{lemma:personalization}, e.g.,~\eqref{eq:18} will be written as  $$\left|\left(\widehat{Q}_{ab}(z)/Q_{\pi_1^*(a)\pi_2^*(b)}(z)\right) - 1 \right| \le \varepsilon'_n.$$ 
The same assumptions are made in the analysis of Stages 3 and 4 below. 
\end{remark}

\noindent{\textbf{Analysis of Stage 3:}}
Note that the likelihood function defined in~\eqref{eq:La} depends on the estimated values $\widehat{\R}$ and $\{\widehat{Q}_{ab}\}$ of the model parameters. For ease of analysis, we first ignore the imprecisions of these estimates, and define the \emph{exact} likelihood function $\tl_a(i)$, which depends on the exact values of $\R$ and $\{Q_{ab}\}$, as 
\begin{align}
\tl_a(i) &:= \sum_{a'\in[k_1]} e(\{i\}, \U^{(0)}_{a'})\cdot \log\left(\R_{aa'}/(1-\R_{aa'})\right) \notag \\
&\qquad\quad + \sum_{b \in [k_2]}\sum_{j \in \I_b^{(0)}} \mathbbm{1}\{\UU_{ij} \!\ne\! \mathsf{e} \} \cdot \log Q_{ab}(\UU_{ij}). 
\end{align}

We now consider a specific user $i \in [n]$, which belongs to cluster $\U_a$ for some $a \in [k_1]$. Lemma~\ref{lemma:La} below shows that, with probability $1 - o(1/n)$,  $\tl_a(i)$ is larger than any other likelihood functions $\tl_{\ba}(i)$ by at least $(\epsilon/2)\log n$.  

\begin{lemma}\label{lemma:La}
	Suppose $\mathbf{h} \in \mathcal{G}$ and $(\sigma^{(0)}, \tau^{(0)}) \in \mathcal{G}'$.	If  
	\begin{align}
	mnp \ge \frac{\left[(1+\epsilon) - (I_1/k_1)\right] n\log n}{d_{\U}/k_2}, \label{eq:mnp}
	\end{align}
with probability at least $1 - n^{-(1+\frac{\epsilon}{4})}$, 
	\begin{align}
	\tl_a(i) > \max_{\ba \in [k_1]\setminus\{a\}} \tl_{\ba}(i) + (\epsilon/2)\log n. \label{eq:max}
	\end{align}
\end{lemma}
\begin{proof}
Consider a specific $\ba \ne a$. Under the symmetric setting, the entries in the connectivity matrix $\R$ are either $\alpha_1$ or $\beta_1$; thus, we define $\lambda_1 := \log\!\frac{(1-\beta_1)\alpha_1}{(1-\alpha_1)\beta_1}$ and one can show that
\begin{align}
\tl_a(i) - \tl_{\ba}(i) &= \lambda_1 e(\{i\},\U_a^{(0)}) - \lambda_1 e(\{i\},\U_{\ba}^{(0)}) \notag \\
& + \!\!\sum_{b \in [k_2]}\sum_{j \in \I_b^{(0)}} \mathbbm{1}\{\UU_{ij} \!\ne\! \mathsf{e} \}\log \frac{Q_{ab}(\mathsf{U}_{ij})}{Q_{\ba b}(\mathsf{U}_{ij})}. \label{eq:aba}
\end{align}

For $a,\ba \in [k_1]$, let $\s_{a\ba} := \U_a \cap \U_{\ba}^{(0)}$ be the set of users that belong to cluster $\U_a$ and are classified to $\U_{\ba}^{(0)}$ after Stage 1. By introducing random variables $\{X_k \}_{k=1}^n \stackrel{\text{i.i.d.}}{\sim} \mathrm{Bern}(\alpha_1)$ and $\{Y_k \}_{k=1}^n \stackrel{\text{i.i.d.}}{\sim} \mathrm{Bern}(\beta_1)$, one can rewrite $e(\{i\},\U_a^{(0)}) - e(\{i\},\U_{\ba}^{(0)})$ as
\begin{align}
\sum_{k\in\s_{aa}} X_k + \!\!\!\sum_{k\in\U_a^{(0)}\setminus \U_a } Y_k - \!\!\!\sum_{k\in \U_{\ba}^{(0)}\setminus \U_a} Y_k - \!\sum_{k\in \s_{a\ba}} X_k. \label{eq:xiao1}
\end{align}
For $b,\bb \in [k_2]$, let $\T_{b\bb} := \I_b \cap \I_{\bb}^{(0)}$ be the set of items that belong to cluster $\I_b$ and are classified to $\I_{\bb}^{(0)}$ after Stage 1. By introducing random variables $\{T_{ij} \} \stackrel{\text{i.i.d.}}{\sim} \mathrm{Bern}(p)$ and $\{Z_{ij}^{(ab)} \} \stackrel{\text{i.i.d.}}{\sim} Q_{ab}$, one can rewrite the second part in~\eqref{eq:aba} as 
\begin{align}
\sum_{b \in [k_2]} \Bigg[\sum_{j \in \T_{bb}} \underbrace{T_{ij} \log \frac{Q_{ab}(Z^{ab}_{ij})}{Q_{\ba b}(Z^{ab}_{ij})}}_{:=A^{ab}_{ij}} + \sum_{\bb \ne b} \sum_{j\in \T_{\bb b}} T_{ij} \log \frac{Q_{a\bb}(Z^{a\bb}_{ij})}{Q_{\ba \bb}(Z^{a\bb}_{ij})} \Bigg]. \label{eq:xiao2}
\end{align}	
 
Representing $\tl_{aa'}(i)$ in terms of~\eqref{eq:xiao1} and~\eqref{eq:xiao2}, and applying the \emph{Chernoff bound}
$\PP(X>\kappa) \le \min_{t>0} e^{-t\kappa}\cdot \E(e^{tX})$ 
with $t = 1/2$, we then have
	\begin{align}
	&\PP\left( \tl_{a}(i) - \tl_{\ba}(i) < (\epsilon/2)\log n\right) \notag \\
	&\le \exp\Big[(\epsilon/2)\log n -(1-o(1))I_1(\log n)/k_1  \notag \\
	&\qquad\qquad- (1-o(1))\sum_{b\in[k_2]} mpH^2(Q_{ab},Q_{\ba b})/k_2 \Big] \label{eq:aii} \\
	&\le n^{-(1+\frac{\epsilon}{4})}, \label{eq:deng2}
	\end{align}
where~\eqref{eq:aii} follows from the facts that (i) $\E(e^{-\frac{1}{2}A^{(ab)}_{ij}}) = 1 - pH^2(Q_{ab},Q_{\ba b})$, (ii) $\E(e^{-\frac{1}{2}(Y_k - X_k)}) = I_1 \log(n)/n$,  and (iii) the misclassification proportions in $\sigma^{(0)}$ and $\tau^{(0)}$ are  negligible, i.e., $|\s_{aa}| \ge (\frac{1}{k_1}-o(1))n$, $|\U_{\ba}^{(0)}\setminus \U_a| \ge (\frac{1}{k_1}-o(1))n$, and $|\T_{bb}| \ge (\frac{1}{k_2}-o(1))m$. Eqn.~\eqref{eq:deng2} holds since $mnp$ satisfies~\eqref{eq:mnp}  and $d_{\U} \le \sum_{b\in[k_2]}H^2(Q_{ab},Q_{\ba b})$. Finally, by taking a union bound over all the clusters $\U_{\ba}$ such that $\ba \in [k_1]\setminus\{a \}$, we complete the proof of Lemma~\ref{lemma:La}.
\end{proof}

Note that Lemma~\ref{lemma:La} is for a specific user $i \in [n]$. Taking a union bonud over the $n$ users yields that with probability $1-o(1)$, all the users $i \in [n]$ satisfy
\begin{align}
\tl_{\sigma(i)}(i) > \max_{\ba \in [k_1]\setminus\{\sigma(i)\}} \tl_{\ba}(i) + (\epsilon/2)\log n, \label{eq:max2}
\end{align}   
where $\sigma(i)$ is the user cluster that user $i$ belongs to. 

Finally, it is shown in Lemma~\ref{claim:con2} below that the difference between the exact likelihood function $\tl_a(i)$ and the original likelihood function $L_a(i)$ is negligible.
\begin{lemma} \label{claim:con2}
	With probability $1- o(1)$, there exists a sequence $\xi_n \in \Omega(\max\{\varepsilon_n, \varepsilon'_n\}) \cap o(1)$ such that for all $a \in [k_1]$ and all users $i \in [n]$,
	$\big| L_a(i) - \tl_a(i) \big| \le \xi_n \log n.$
\end{lemma}
The proof of Lemma~\ref{claim:con2} can be found in Appendix~\ref{appendix:con2}. 
Combining~\eqref{eq:max2} and Lemma~\ref{claim:con2} via the triangle inequality, we have that all the users satisfy $L_{\sigma(i)}(i) > \max_{\ba \in [k_1]\setminus\{\sigma(i)\}} L_{\ba}(i)$. This ensures the success of Stage 3, i.e., $\widehat{\sigma}(i) = \sigma(i), \forall i \in [n]$.

\noindent{\textbf{Analysis of Stage 4:}} 
The analysis of Stage 4 is similar to that of Stage 3. Lemma~\ref{lemma:movie} below states that all the $m$ items can be classified into the correct cluster when $mnp$ satisfies~\eqref{eq:mnp2}.
\begin{lemma}\label{lemma:movie}
	Suppose $\mathbf{h} \in \mathcal{G}$ and $(\sigma^{(0)}, \tau^{(0)}) \in \mathcal{G}'$.	If 
	\begin{align}
	mnp \ge \frac{((1+\epsilon) - (I_{2}/k_2)) m\log m}{d_{\mathcal{I}}/k_1}, \label{eq:mnp2}
	\end{align}
	with probability $1-o(1)$, all the items $j \in [m]$ satisfy 
	\begin{align}
L'_{\tau(j)}(j) > \max_{\bar{b} \in [k_2]\setminus\{\tau(j)\}} L'_{\bar{b}}(j).
	\end{align}
\end{lemma}

Finally, based on the outputs $\{\widehat{\U}_a\}_{a\in [k_1]}$ and $\{\widehat{\I_b}\}_{b\in[k_2]}$ of {\sc Mc2g}, one can recover the nominal matrix $\widehat{\mathsf{N}}$ via \emph{majority voting}. Specifically, for $a\in[k_1]$ and $b\in[k_2]$, we define $u_{ab}:= \argmax_{z \in \mathcal{Z}} \sum_{i\in \widehat{\U}_a}\sum_{j\in \widehat{\I}_b}\mathbbm{1}\{\UU_{ij}=z \}$, and we then set   
\begin{align}
\widehat{\mathsf{N}}_{ij}= u_{ab}, \quad \mathrm{if }\  i\in \widehat{\U}_a, j\in \widehat{\I}_b.  \label{eq:nominal2}
\end{align}
The correctness of~\eqref{eq:nominal2} follows from the fact that $\sum_{i\in \widehat{\U}_a}\sum_{j\in \widehat{\I}_b}\mathbbm{1}\{\UU_{ij}=z \} \approx mnQ_{ab}(z)/(k_1k_2)$ (which is due to the Chernoff bound).

\subsection{The Overall Success Probability} \label{sec:error}
Let $E_{\text{suc}}$ be the event that {\sc Mc2g} exactly recovers the nominal matrix. From the analyses of Stages 2--4, we know that for any $\mathbf{h} \in \mathcal{G}$ and $(\sigma^{(0)}, \tau^{(0)}) \in \mathcal{G}'$, 
\begin{align}
\PP(E_{\text{suc}}|\mathbf{h},(\sigma^{(0)}, \tau^{(0)})) \ge 1- o(1), \label{eq:egg1}
\end{align}
where~\eqref{eq:egg1} is uniform in $\mathbf{h} \in \mathcal{G}$ and $(\sigma^{(0)}, \tau^{(0)}) \in \mathcal{G}'$.
Therefore, the overall success probability is lower bounded as 
\begin{align}
\PP(E_{\text{suc}}) &= \sum_{\mathbf{h} \in \mathcal{G}} \PP(\mathbf{h})\PP(E_{\text{suc}}|\mathbf{h}) + \sum_{\mathbf{h} \in \mathcal{B}_{\mathbf{h}}} \PP(\mathbf{h})\PP(E_{\text{suc}}|\mathbf{h}) \notag \\
&\ge  \sum_{\mathbf{h} \in \mathcal{G}} \PP(\mathbf{h}) \sum_{(\sigma^{(0)}, \tau^{(0)}) \in \mathcal{G}'} \PP((\sigma^{(0)}, \tau^{(0)})|\mathbf{h}) \notag \\
&\qquad\qquad\qquad\qquad \times \PP(E_{\text{suc}}|\mathbf{h},(\sigma^{(0)}, \tau^{(0)})) \notag \\
&\ge (1 - o(1)) \sum_{\mathbf{h} \in \mathcal{G}} \PP(\mathbf{h}) \sum_{(\sigma^{(0)}, \tau^{(0)}) \in \mathcal{G}'} \PP((\sigma^{(0)}, \tau^{(0)})|\mathbf{h}) \label{eq:egg2}\\
&\ge (1 - o(1)) (1 - \sqrt{\epsilon_n})^2 \label{eq:egg3} \\
&= (1 - o(1)), \notag 
\end{align}
where~\eqref{eq:egg2} is due to~\eqref{eq:egg1}, and~\eqref{eq:egg3} follows from~\eqref{eq:72}.

\section{Proof sketch of Theorem~\ref{thm:2}} \label{sec:proof2}
The proof techniques used for Theorem~\ref{thm:2} is a generalization of the techniques used in~\cite[Sec. IV-B]{eric2021}, thus we only provide a proof sketch here. The key idea is to first show that the maximum likelihood (ML) estimator $\phi_{\mathrm{ML}}$ is the optimal estimator (as proved in~\cite[Eqn.~(33)]{eric2021}), and then analyze the error probability with respect to $\phi_{\mathrm{ML}}$---the crux of the analysis is to focus on a subset of of events that are most likely to induce errors, and to prove the tightness of the Chernoff bound.   

To analyze $\phi_{\text{ML}}$, we first show that under the model parameter  $(\sigma,\tau,\N)$ (where a single parameter $\xi$ is used to be the abbreviation of $(\sigma,\tau,\N)$ in the following), the log-likelihood of observing $(\UU, G_1, G_2)$ is
\begin{align}
&\log\PP_{\xi}(\UU,G_1,G_2) \!=\! e_1^{\sigma}\log\!\frac{\beta_1(1\!-\!\alpha_1)}{\alpha_1(1\!-\!\beta_1)} \!+\! e_2^{\tau}\log\!\frac{\beta_2(1\!-\!\alpha_2)}{\alpha_2(1\!-\!\beta_2)} \notag \\
&\quad + \sum_{a \in [k_1]}\sum_{b\in [k_2]}\sum_{z \in \mathcal{Z}} |\mathcal{D}_{ab}^z(\xi)|\cdot \log Q_{ab}(z) + C_0,
\end{align}
where $e_1^{\sigma}$ is the number of inter-cluster edges in $G_1$ with respect to $\sigma$; $e_2^{\tau}$ is the number of inter-cluster edges in $G_2$ with respect to $\tau$; $\mathcal{D}_{ab}^z(\xi) = \{(i,j) \in [n] \times[m]: \sigma(i) = a, \tau(j) = b, \UU_{ij}=z \}$ is the number of observed ratings $z$ corresponding to user cluster $\U_a$ and item cluster $\I_b$; and $C_0$ is a constant that is independent of $(\sigma,\tau,\N)$. 

Suppose $\xi$ is the ground truth that governs the model from now on, and note that the ML estimator $\phi_{\mathrm{ML}}$ succeeds if $\xi$ is the most likely model parameter in $\Xi$ conditioned on the observation $(\UU,G_1,G_2)$, i.e., $\log\PP_{\xi}(\UU,G_1,G_2)$ is larger than any other $\log\PP_{\xi'}(\UU,G_1,G_2)$ for $\xi' \in \Xi \setminus \{\xi\}$. In fact, what we show in the converse proof is that when $mnp$ is less than the bound in~\eqref{eq:converse1}, with high probability there exists another model parameter $\xi' \in \Xi \setminus \{\xi\}$ such that the likelihood $\log\PP_{\xi'}(\UU,G_1,G_2)$ achieves the maximum.

Specifically, let $\xi' \ne \xi$ be a model parameter that is identical to $\xi$ except that its first component $\sigma'$ differs from $\sigma$ by only two labels, i.e., $\sum_{i\in[n]} \mathbbm{1}\{\sigma'(i) \ne \sigma(i)\} = 2$. As the distinction between $\xi'$ and $\xi$ is small, the probability that $\log\PP_{\xi'}(\UU,G_1,G_2) \ge  \log\PP_{\xi}(\UU,G_1,G_2)$ turns out to be relatively large, which is at least  
\begin{align}
\frac{1}{4}\exp\left\{-(1\!+\!o(1))\frac{2I_1(\log n)}{k_1} \!-\! (1\!+\!o(1))\frac{2mpd_{\U}}{k_2} \right\} \label{eq:xi}
\end{align}
due to the tightness of the Chernoff bound (which can be proved by generalizing~\cite[Lemma 2]{eric2021}). In fact, one can find a subset $\Xi_0 \subseteq \Xi$ of model parameters such that $|\Xi_0| = \Theta(n)$ and each element in $\xi_0 \in \Xi_0$ satisfies~\eqref{eq:xi} (i.e., the probability that $\xi_0$ induces an error is relatively large). This, together with the assumption that $mnp < k_2\left[\frac{1-\epsilon}{2} \!-\! \frac{I_{1}}{k_1}\right]\! n\log n/d_{\U}$, eventually implies that with probability approaching one, there exists at least one $\xi_0 \in \Xi_0$ such that $\log\PP_{\xi_0}(\UU,G_1,G_2) \ge  \log\PP_{\xi}(\UU,G_1,G_2)$. Thus, the ML estimator fails.   

In a similar and symmetric fashion, one can show that the ML estimator fails with probability approaching one,  when $mnp < k_1\left[\frac{1-\epsilon}{2} \!-\! \frac{I_{2}}{k_2}\right]\! m\log m/d_{\I}$. This completes the proof of the converse part.


\section{Experiments} \label{sec:experiment}

In this section, we apply the simplified version of {\sc Mc2g} mentioned in Remark~\ref{remark:simplified} (without the information splitting step), as the sizes of the graphs $m$ and $n$ cannot be made arbitrarily large in the experiments.\footnote{As discussed in Remark~\ref{remark:simplified}, the information splitting method is merely for the purpose of analysis, and the first part of the graphs $(G_1^\mathrm{a}, G_2^\mathrm{a})$ turns out to be too sparse to achieve weak recovery of clusters when $m$ and $n$ are not large enough.} That is, the four stages are applied to the  \emph{fully-observed} graphs $(G_1,G_2)$. While this implementation is slightly different  from the original algorithm as described in Algorithm~\ref{algorithm:1}, its  empirical performance nonetheless demonstrates a keen agreement with the theoretical guarantee for the original {\sc Mc2g} in Theorem~\ref{thm:algorithm} (as shown in Section~\ref{sec:syn} below).

\begin{figure}[t]
	\centering
	\includegraphics[width=0.4\textwidth]{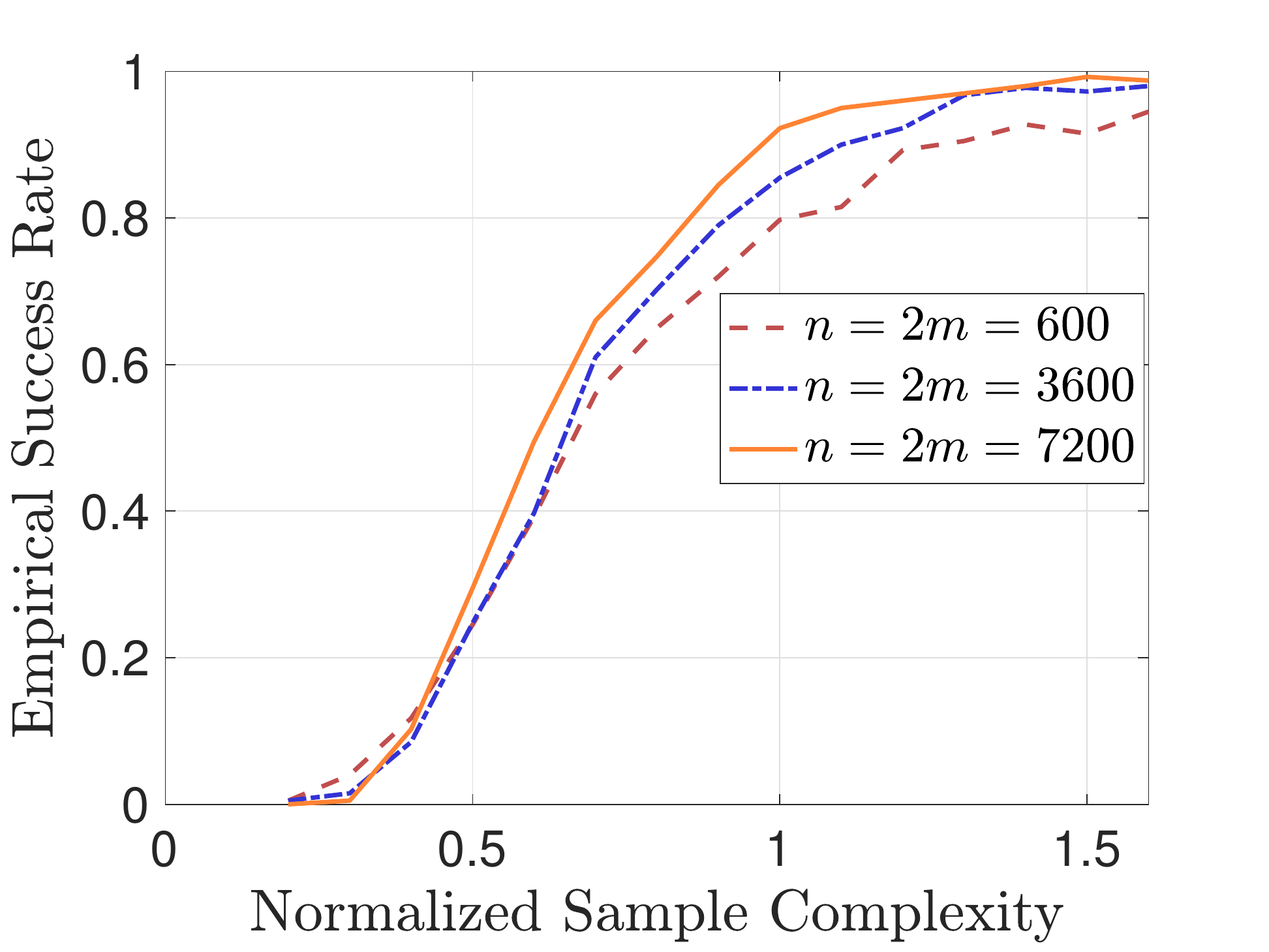}
	\caption{The empirical success rate (over 400 trials) vs. the normalized sample complexity under the setting described in Section~\ref{sec:syn}. }
	\label{fig:syn1}
\end{figure}

\subsection{Verification  of Theorem~\ref{thm:algorithm} on synthetic data} \label{sec:syn}
We verify the theoretical guarantee provided in  Theorem~\ref{thm:algorithm} on a synthetic dataset generated according to a symmetric setting described as follows. The setting contains $k_1 = 3$ user clusters $\{\U_1, \U_2, \U_3\}$, $k_2 = 4$ item clusters $\{\I_1, \I_2, \I_3, \I_4\}$, nominal ratings $\{z_{ab}\}$ satisfying
\begin{align*}
&z_{11}=5, \ \ z_{12}=1, \ \ z_{13}=4, \ \ z_{14}=2, \\
&z_{21}=2, \ \ z_{22}=4, \ \ z_{23}=5,\ \ z_{24}=1,\\
&z_{31}=3, \ \ z_{32}=2, \ \ z_{33}=5,\ \ z_{34}=5,
\end{align*}
which are chosen from $\mathcal{Z}=\{1,2,3,4,5\}$, and personalization distributions 
\begin{align}
Q_{V|Z}(v|z) = \begin{cases} 0.6,  &\text{ if } \  v = z, \\
0.1, &\text{ if } \ v \ne z.
\end{cases}
\end{align}
We set $n = 2m$, and both $I_1$ and $I_2$ (the qualities of social and item graphs) to $2$. Fig.~\ref{fig:syn1} shows the \emph{empirical success rate}  as a function of the \emph{normalized sample complexity} for three different values of $m$ and $n$. The empirical success rate is averaged over $400$ random trials, and the normalized sample complexity is defined (according to Theorem~\ref{thm:algorithm}) as $mnp$ divided by
\begin{align}
\max\left\{\frac{(1-(I_1/k_1))n\log n}{d_{\mathcal{U}}/k_2}, \frac{(1-(I_2/k_2))m\log m}{d_{\mathcal{I}}/k_1} \right\}.
\end{align}
It can be seen from Fig.~\ref{fig:syn1} that as the normalized sample complexity increases, the empirical success rate also increases and becomes close to one when the normalized sample complexity exceeds one (corresponding exactly to the condition for {\sc Mc2g} to succeed).

\subsection{Comparing {\sc Mc2g}  with other algorithms on synthetic data} \label{sec:syn_compare}

Next, we compare {\sc Mc2g} to several existing recommendation algorithms that leverage graph side information on another synthetic dataset. The competitors include the matrix factorization with social regularization (SoReg)~\cite{ma2011recommender}, and a spectral clustering method with local refinements using {\em only} the social graph or {\em only} the item graph as side information by Ahn {\em et al.}~\cite{ahn2018binary}.  In fact, we have also compared our algorithm to other matrix completion algorithms such as biased matrix factorization (MF)~\cite{koren2008factorization} and TrustSVD~\cite{guo2015trustsvd}, but they did not perform as well as the competitors we chose and {\sc Mc2g}.  This synthetic dataset is simpler compared to the one in Section~\ref{sec:syn}, as we need to choose the ratings $\mathcal{Z}$ to be \emph{binary} (as other competing algorithms are amenable only to binary ratings). It contains $n=3000$ users partitioned into two user clusters, $m = 3000$ items partitioned into three item clusters, and we set the qualities of graphs $I_1 = 1.5$ and $I_2 = 2$, as well as the nominal ratings to be $z_{11} = 0$, $z_{12} = 1$, $z_{13} = 0$, $z_{21} = 0$, $z_{22} = 0$, $z_{23} = 1$. The personalization distributions are modelled as additive Bern$(0.25)$ noise, i.e., $Q_{V|Z}(v|z)$ equals $0.75$ if $v = z$, and equals $0.25$ otherwise.

To ensure that the comparisons are fair, we quantize the outputs of the other algorithms to be $\{0,1\}$-valued. We measure the performances using the {\em mean absolute error} (MAE) 
\begin{align}
\mathrm{MAE} := \sum_{i=1}^n\sum_{j = 1}^m \frac{|\widehat{\mathsf{N}}_{ij} - \mathsf{N}_{ij}|}{mn}.
\end{align} Fig.~\ref{fig:syn_comparison} shows the MAE (averaged over $100$ random trials) of each algorithm when $p \in [0.001, 0.01]$. It is clear that {\sc Mc2g} is orders of magnitude better than the competing algorithms in terms of the MAEs for this synthetic dataset.  

\begin{figure}[t]
	\centering
	\includegraphics[width=0.4\textwidth]{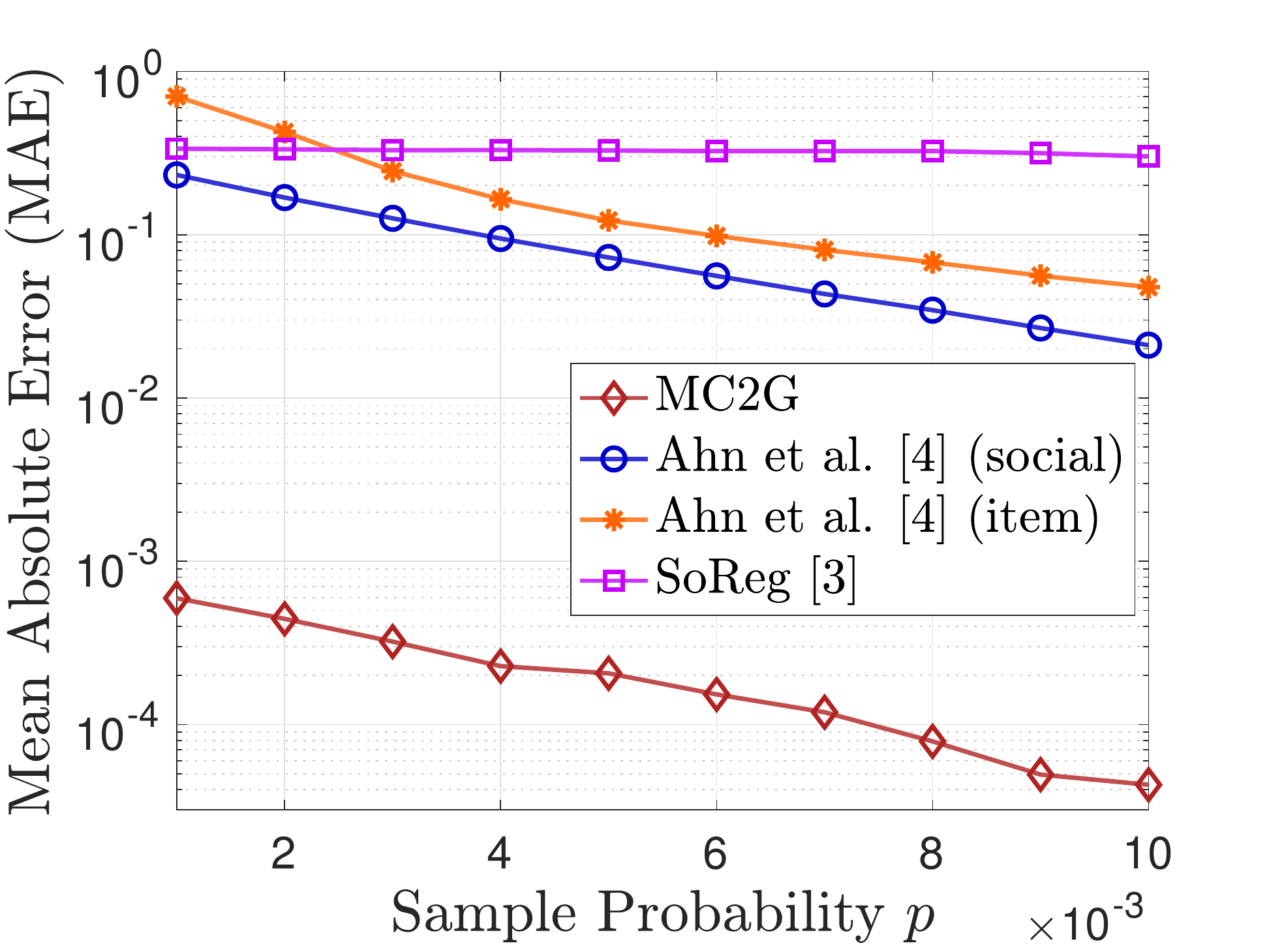}
	\caption{Comparisons of MAEs of different algorithms under the synthetic setting described in Section~\ref{sec:syn_compare}.}
	\label{fig:syn_comparison}
\end{figure}

\subsection{Comparing {\sc Mc2g} with other algorithms on real graphs} \label{sec:real_compare}
To demonstrate that {\sc Mc2g} is amenable to datasets with real graphs, we applied it to real social and item similarity graphs.
\begin{itemize}[wide, labelwidth=!, labelindent=0pt]
\item We adopt the LastFM social network~\cite{rozemberczki2020characteristic} (collected in March 2020) as the social graph. Each node is a LastFM user, while each edge represents mutual follower relationships between users. We sub-sample $n = 1806$ users from the LastFM social network. These users are partitioned into four clusters with sizes $(n_1,n_2,n_3,n_4) = (497,327,552,430)$, and the empirical connection probabilities are 
\begin{align*}
\R= \begin{bmatrix}
9.07 & 0.10 & 0.05 & 0.14 \\
0.10 & 18.3 & 0.17 & 0.25 \\
0.05 & 0.17 & 9.3  & 0.21 \\
0.14 & 0.25 & 0.21  & 22.0 \\
\end{bmatrix} \times 10^{-3}.
\end{align*}

\item We adopt the political blogs network~\cite{adamic2005political} as the item similarity graph. Each node represents a blog that is either liberal-leaning or conservative-leaning, and each edge represents a link between two blogs. This network contains $m = 1222$ blogs which are partitioned into two clusters with sizes $(m_1,m_2)=(586,636)$, and the empirical connection probabilities are
\begin{align*}
&\R' = \begin{bmatrix}
42.6 & 4.2 \\
4.2 & 38.8
\end{bmatrix} \times 10^{-3}.
\end{align*}
\end{itemize}
We also choose $\mathcal{Z} = \{0,1\}$, set the nominal ratings to be 
\begin{align*}
&z_{11} = 0, \ \ z_{21} = 0, \ \ z_{31} = 1, \ \ z_{41} = 1, \\
&z_{12} = 0, \ \ z_{22} = 1, \ \ z_{32} = 0, \ \ z_{42} = 1,
\end{align*}
and model the personalization distributions as additive Bern$(0.1)$ noise. 
The personalized ratings matrix $V$ is then synthesized based on the user and item clusters, nominal ratings, and personalization distributions described above.

\begin{figure}[t]
	\centering
	\includegraphics[width=0.4\textwidth]{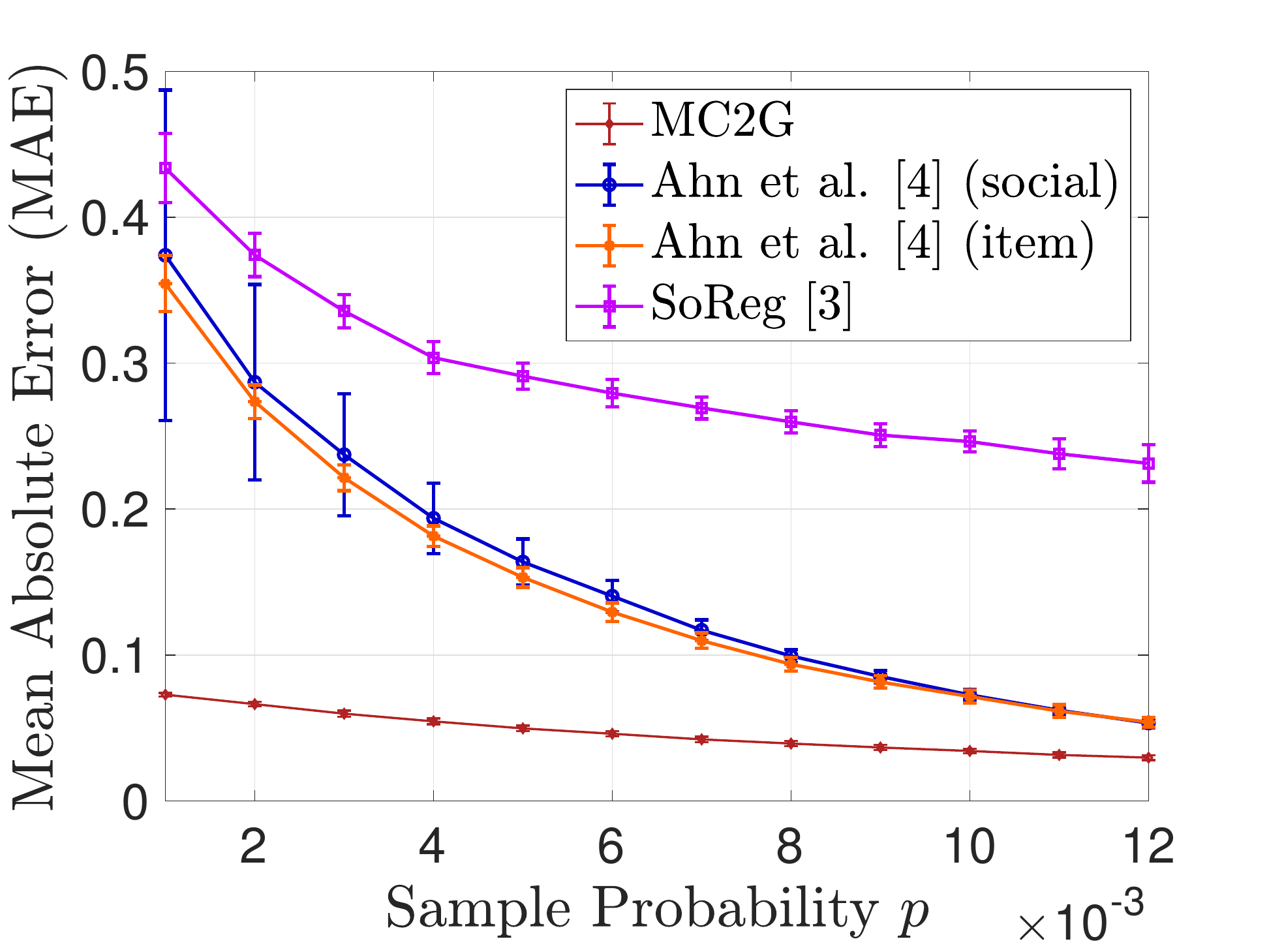}
	\caption{Comparisons of MAEs of different algorithms under the semi-real setting described in Section~\ref{sec:real_compare}, where we adopt the LastFM social network and political blog networks as social and item similarity graphs respectively. The length of each errorbar above and below each data point represents the standard deviations across the $100$ independent trials.}
	\label{fig:real_comparison}
\end{figure}

We compare {\sc Mc2g} to the algorithms introduced in~\ref{sec:syn_compare} on this semi-real dataset (real social and item similarity graphs with synthetic ratings).  Fig.~\ref{fig:real_comparison} shows the MAE (averaged over $100$ trials) of each algorithm when $p \in [0.001,0.012]$. Clearly, {\sc Mc2g} is superior to the other algorithms, and the advantage is more significant when the sample probability $p$ is small. In addition, the errorbars above and below each data point (representing one standard deviation) for {\sc Mc2g} are fairly small, demonstrating the statistical robustness of {\sc Mc2g}. The average running time (in seconds) of each algorithm, when $p = 0.01$, is as follows, showing that the running time of {\sc Mc2g} is commensurate with its prediction abilities.   

\vspace{5pt}

\begin{small}
	\begin{tabular}{cccc}
		\toprule
		{\sc Mc2g}      & Ahn et al.~\cite{ahn2018binary} (social) & \cite{ahn2018binary} (item) & SoReg~\cite{ma2011recommender}   \\
		\midrule
$5.199s$ & $4.647s$ & $0.973s$ & $0.010s$   \\
		\bottomrule
	\end{tabular}
\end{small}

\vspace{5pt}

It is worth mentioning that the running times of {\sc Mc2g} and the algorithm in Ahn et al.~\cite{ahn2018binary} with either a social or an item graph are dominated by the spectral initialization steps---this is the reason why the running times of {\sc Mc2g} are longer than the algorithm in~\cite{ahn2018binary} (as {\sc Mc2g} performs spectral clustering for \emph{both} social and item graphs). SoReg~\cite{ma2011recommender} runs faster since it does not perform spectral clustering; however, its performance is rather poor as can be seen from Figs.~\ref{fig:syn_comparison} and~\ref{fig:real_comparison}.

\appendices
\section{Proof of Lemma~\ref{lemma:new}}\label{appendix:1}

Consider the process of first generating a sub-graph $H_1^{\mathrm{a}}$ and then generating a sub-SBM $G_1^\mathrm{a}$ on the sub-graph $H_1^{\mathrm{a}}$. The probability that an edge $E_{ii'}$ (connecting nodes $i$ and $i'$) appears in $G_1^\mathrm{a}$ equals\footnote{Specifically, the probability that an edge $E_{ii'}$ appears in $G_1^\mathrm{a}$ is equal to the probability of $E_{ii'}$ belonging to $H_1^{\mathrm{a}}$ multiplied by the probability of generating $E_{ii'}$ in the sub-SBM $G_1^\mathrm{a}$.} $1/\sqrt{\log n}$ multiplied by $\alpha_1$ or $\beta_1$ (depending on whether $i$ and $i'$ are in the same community). Thus, a key observation is that the aforementioned process  is equivalent to generating $\widetilde{G}_1$ directly. By this observation and recalling  that a spectral clustering method running on $\widetilde{G}_1$ ensures $l_1(\sigma^{(0)}, \sigma) \le \eta_n$ with probability at least $1 - \epsilon_n$~\cite[Theorem 6]{NIPS2016_a8849b05}, we have    
\begin{align}
\sum_{h_1^{\mathrm{a}}} \PP(H_1^{\mathrm{a}} = h_1^{\mathrm{a}}) P_{\text{suc}}(h_1^{\mathrm{a}}) \ge 1 - \epsilon_n, \label{eq:contri}
\end{align}
where $P_{\text{suc}}(h_1^{\mathrm{a}})$ is the probability that a spectral clustering method running on $G_1^\mathrm{a}$ (which depends on $h_1^{\mathrm{a}}$) ensures $l_1(\sigma^{(0)}, \sigma) \le \eta_n$.
Let $\mathcal{H}_1^{\mathrm{a},\mathcal{G}}$ and $\mathcal{H}_1^{\mathrm{a},\mathcal{B}}$ respectively be the sets of good and bad sub-graphs $h_1^{\mathrm{a}}$ . Suppose the probability of generating a good sub-graph  (i.e., $h_1^{\mathrm{a}} \in \mathcal{H}_1^{\mathrm{a},\mathcal{G}}$) is less than $1 - \sqrt{\epsilon_n}$. Then, by the definition of the good sub-graphs $h_1^{\mathrm{a}}$, 
\begin{align*}
&\sum_{h_1^{\mathrm{a}}} \PP(H_1^{\mathrm{a}} = h_1^{\mathrm{a}}) P_{\text{suc}}(h_1^{\mathrm{a}}) \\
&< \sum_{h_1^{\mathrm{a}} \in \mathcal{H}_1^{\mathrm{a},\mathcal{G}}} \PP(H_1^{\mathrm{a}} = h_1^{\mathrm{a}}) +  \sum_{h_1^{\mathrm{a}} \in \mathcal{H}_1^{\mathrm{a},\mathcal{B}}} \PP(H_1^{\mathrm{a}} = h_1^{\mathrm{a}}) (1- \sqrt{\epsilon_n}) \\
&= \sum_{h_1^{\mathrm{a}} \in \mathcal{H}_1^{\mathrm{a},\mathcal{G}}}\!\! \PP(H_1^{\mathrm{a}} \!=\! h_1^{\mathrm{a}}) + (1\!-\! \sqrt{\epsilon_n})\Bigg(1 \!-\!\!\! \sum_{h_1^{\mathrm{a}} \in \mathcal{H}_1^{\mathrm{a},\mathcal{G}}}\!\! \PP(H_1^{\mathrm{a}} = h_1^{\mathrm{a}})\Bigg) \\
&< 1- \epsilon_n,
\end{align*}
which yields a contradiction to~\eqref{eq:contri}. Thus, we conclude that with probability at least $1 - \sqrt{\epsilon_n}$ over the generation of $H_1^{\mathrm{a}}$, the randomly generated $h_1^{\mathrm{a}}$ is a good sub-graph.

For each user node $i \in [n]$, the expected degree of $i$ in $H_1^{\mathrm{a}}$ is $(n-1)/\sqrt{\log n}$. By applying the multiplicative form of the Chernoff bound, one can show that with probability at least $1 - \exp(-\Theta(n/\sqrt{\log n}))$, the degree of $i$ in the randomly generated sub-graph $h_1^{\mathrm{a}}$ is at most $2n/\sqrt{\log n}$. A union bound over all user nodes guarantees that, with high probability, the degrees of \emph{all} the nodes in $h_1^{\mathrm{a}}$ are at most $2n/\sqrt{\log n}$, which further implies the sub-graph $h_2^\mathrm{b}$ is good. 
Finally, applying a union bound implies that that with high probability, both $h_1^{\mathrm{a}}$ and $h_1^{\mathrm{b}}$ are good sub-graphs simultaneously.  

In a similar manner, we can also prove the analogous statements for $H_2^{\mathrm{a}}$ and $H_2^\mathrm{b}$.

\section{Proof of Lemma~\ref{lemma:alpha}} \label{appendix:alpha}
First recall the definitions of $\s_{a\ba}$ and $\T_{b\bb}$ in Section~\ref{sec:proof} (Analysis of Stage 3). As it is assumed that $\mathbf{h} \in \mathcal{G}$ and $(\sigma^{(0)}, \tau^{(0)}) \in \mathcal{G}'$, we know that (i) $|\s_{a\ba}| \ge ((1/k_1)-\eta_n)n$ when $a = \ba$, and $|\s_{a\ba}| \le \eta_n n$ when $a \ne \ba$; and (ii) $|\T_{b \bb}| \ge ((1/k_2)-\gamma_n)m$ when $b = \bb$, and $|\T_{b\bb}| \le \gamma_n m$ when $b \ne \bb$.

We first analyze the estimates $\{\widehat{\mathsf{B}}_{aa}\}_{a\in[k_1]}$ in Eqn.~\eqref{eq:est1}. By letting $\{X_{k}\} \stackrel{\text{i.i.d.}}{\sim} \text{Bern}(\alpha_1)$ and $\{Y_{k}\} \stackrel{\text{i.i.d.}}{\sim} \text{Bern}(\beta_1)$, we have
$e(\U_a^{(0)},\U_a^{(0)}) = \sum_{k=1}^{B_1} X_k + \sum_{k=1}^{B_2} Y_k,$ 
where $B_1 := \sum_{\ba \in[k_1]} \binom{|\s_{\ba \ba}|}{2}$ and $B_2 := \binom{|\U_{a}^{(0)}|}{2} - B_1$. Note  that 
\begin{align*}
\mu_{\mathsf{B}_{aa}} \!:= \E[ e(\U_a^{(0)},\U_a^{(0)})] \!\le \alpha_1 \binom{|\U_{a}^{(0)}|}{2}.
\end{align*}
On the other hand, since the degree of any nodes in $h_1^{\mathrm{b}}$ is at least $n(1 - 2/\sqrt{\log n})$ (or equivalently, the number of non-edges of any nodes is at most $2n/\sqrt{\log n}$), we know that
\begin{align*}
&e(\U_a^{(0)}\!,\U_a^{(0)}) \!\ge\!\!\!\! \sum_{k=1}^{B_1 - \frac{n}{2}\frac{2n}{\sqrt{\log n}}} \!\!\!X_k, \ \text{and} \  \mu_{\mathsf{B}_{aa}} \!\ge\!\! \left[B_1 \!-\! \frac{n^2}{\sqrt{\log n}}\right]\alpha_1.
\end{align*}
Applying the multiplicative form of the Chernoff bound yields that for any $\delta \in (0,1)$, with probability at least $1 - 2\exp(-\delta^2 \mu_{\mathsf{B}_{aa}}/3)$, the numerator $e(\U_a^{(0)},\U_a^{(0)})$ satisfies 
$$(1-\delta)\!\left(B_1 \!-\! \frac{n^2}{\sqrt{\log n}}\right)\alpha_1 \le e(\U_a^{(0)},\U_a^{(0)}) \le (1+\delta)\alpha_1 \binom{|\U_a^{(0)}|}{2}.$$
As the estimate $\widehat{\mathsf{B}}_{aa}=e(\U_a^{(0)},\U_a^{(0)})/\binom{|\U_a^{(0)}|}{2}$, we then have 
\begin{align*}
\left(1 - \delta - c_1\eta'_n - \frac{c_2}{\sqrt{\log n}}\right) \alpha_1 \le \widehat{\mathsf{B}}_{aa} \le  \left(1+\delta\right) \alpha_1.
\end{align*}
for some constant $c_1, c_2 > 0$. By choosing $\delta = 1/\sqrt{\log n}$, we complete the proof for $\widehat{\mathsf{B}}_{aa}$.

The analyses of other estimates $\{\widehat{\mathsf{B}}_{aa'}\}_{a\ne a'}$, $\{\widehat{\R}_{bb}\}_{b\in[k_2]}$, and $\{\widehat{\R}_{bb'}\}_{b\ne b'}$  are similar, thus we omit them for brevity (except that we need to replace $\eta_n$ by $\gamma_n$ for $\{\widehat{\R}_{bb}\}_{b\in[k_2]}$, and $\{\widehat{\R}_{bb'}\}_{b\ne b'}$).  Therefore, one can find a sequence $\varepsilon_n \in \Omega(\max\{\gamma_n,\eta_n,1/\sqrt{\log n} \}) \cap o(1)$ such that~\eqref{eq:alpha2} holds.

\section{Proof of Lemma~\ref{lemma:personalization}} \label{appendix:personalization}
Let us recall the definition of $\widehat{Q}_{ab}(z)$ in~\eqref{eq:eq}, in which the numerator takes the form
\begin{align}
  |\mathcal{Q}_{ab}^z| &= \sum_{i \in \U^{(0)}_a} \sum_{j \in \I^{(0)}_b} \mathbbm{1}\left\{\UU_{ij} =z \right\}. \label{eq:ess}
 \end{align}
Let $\{T_{ij}\} \stackrel{\text{i.i.d.}}{\sim} \mathrm{Bern}(p)$, $\{Z^{ab}_{ij}\} \stackrel{\text{i.i.d.}}{\sim} Q_{ab}$  for all $a\in[k_1]$ and $b \in[k_2]$. Thus, the RHS of~\eqref{eq:ess} can be rewritten as 
\begin{align*} 
\sum_{i\in\s_{aa}}\sum_{j\in\T_{bb}}\! T_{ij}\mathbbm{1}(Z^{ab}_{ij}=z) + \!\sum_{\bar{a}\ne a}\sum_{\bar{b}\ne b} \sum_{i \in \s_{\bar{a}a}} \sum_{j \in \T_{\bar{b}b}}\! T_{ij}\mathbbm{1}(Z^{\bar{a}\bar{b}}_{ij}=z).
\end{align*}
Note that the number of summands in the first term
$$|i\in \s_{aa}|\cdot |j\in \T_{bb}| \ge  \left[\left((n/k_1)-\eta_n n\right)\left((m/k_2)-\gamma_n m\right)\right] := L.$$
Thus, the expectation of $|\mathcal{Q}_{ab}^z|$ satisfies
\begin{align*}
&\E(|\mathcal{Q}_{ab}^z|) \ge L \cdot \E(T_{ij}\mathbbm{1}(Z^{ab}_{ij}=z)) \ge LpQ_{ab}(z), \quad \mathrm{and}\\
&\E(|\mathcal{Q}_{ab}^z|) \le L\cdot \E(T_{ij}\mathbbm{1}(Z_{ij}=z)) + (mn/(k_1k_2)-L)\E(T_{ij}),
\end{align*}
where the upper bound is due to the fact that $\mathbbm{1}(Z^{\bar{a}\bar{b}}_{ij}=z) \le 1$. Applying the Chernoff bound yields that with probability $1 - \exp\left(-\Theta(\delta_n^2 \E(|\mathcal{Q}_{ab}^z|))\right)$, for all $z \in \mathcal{Z}$, 
\begin{align}
&|\mathcal{Q}_{ab}^z|  \ge  (1-\delta)L pQ_{ab}(z), \quad \mathrm{and} \\
&|\mathcal{Q}_{ab}^z| \le  (1+\delta)(1+\Theta(\max\{\eta_n, \gamma_n\}))L pQ_{ab}(z),
\end{align}
where $\delta \in (0,1)$. Choosing $\delta = 1/\sqrt{\log n}$, we ensure that with probability $1 - o(1)$, for all $z \in \mathcal{Z}$, $a\in[k_1]$, and $b \in [k_2]$,
\begin{align}
\left|\frac{\widehat{Q}_{ab}(z)}{Q_{ab}(z)} - 1\right| = \mathcal{O}(\max\{\eta_n, \gamma_n, 1/\sqrt{\log n}\}).
\end{align}

\section{Proof of Lemma~\ref{claim:con2}} \label{appendix:con2}
First note that 
\begin{align}
&\big| L_a(i) \!-\! \tl_a(i) \big| \!\le \!\!\sum_{a'\in[k_1]} \!\!e(\{i\}, \U^{(0)}_{a'}) \left|\log\frac{\R_{aa'}(1\!-\!\widehat{\R}_{aa'})}{\widehat{\R}_{aa'}(1\!-\!\R_{aa'})}\right| \notag \\
&\qquad \qquad  + \sum_{b \in [k_2]}\sum_{j \in \I_b^{(0)}} \mathbbm{1}\{\UU_{ij} \!\ne\! \mathsf{e} \} \cdot \left|\log \frac{Q_{ab}(\UU_{ij})}{\widehat{Q}_{ab}(\UU_{ij})} \right|.
\end{align}
As $\sum_{a'\in[k_1]} e(\{i\}, \U^{(0)}_{a'})$ represents the degree of user $i$ in the social graph, and its expectation $\mu_i$ satisfies $n\beta_1 \le \mu_i \le n\alpha_1$.  By applying the Chernoff bound, we have that for any $\kappa > 1$,   
\begin{align}
\PP\Bigg(\sum_{a'\in[k_1]} e(\{i\}, \U^{(0)}_{a'}) \ge (1+\kappa)n\alpha_1 \Bigg) \le e^{-\frac{\kappa}{3}n\beta_1}. \label{eq:ch}
\end{align}
We choose $\kappa$ to be a  large enough constant that ensures the RHS of~\eqref{eq:ch} to scale as $o(n^{-1})$. 
Then, by applying the union bound over all the $n$ users, we  have that with probability $1 - o(1)$, all the users $i \in [n]$ satisfy
\begin{align}
\sum_{a'\in[k_1]} e(\{i\}, \U^{(0)}_{a'}) \le  (1+\kappa)n\alpha_1 =  c_1\log n, \label{eq:ha}
\end{align}
for some constant $c_1 > 0$.
Also, note that the term 
$\sum_{b \in [k_2]}\sum_{j \in \I_b^{(0)}} \mathbbm{1}\{\UU_{ij} \!\ne\! \mathsf{e} \}$ corresponds to the number of observed ratings for each user. By a similar analysis (based on the Chernoff bound), one can show that with probability $1 - o(1)$, all the users $i \in [n]$ satisfy
\begin{align}
\sum_{b \in [k_2]}\sum_{j \in \I_b^{(0)}} \mathbbm{1}\{\UU_{ij} \!\ne\! \mathsf{e} \} \le c_2 \log n,\label{eq:ha2}
\end{align} 
for some constant $c_2 > 0$.

Recall from Lemmas~\ref{lemma:alpha} and~\ref{lemma:personalization} that the estimated connection probabilities satisfy $\big|(\widehat{\R}_{aa'} \!-\! \mathsf{B}_{aa'})/\R_{aa'} \big| \le \varepsilon_n$ for all $a,a' \in [k_1]$, and the estimated personalization distribution $\big|(\widehat{Q}_{ab}(z)/Q_{ab}(z))-1\big| \le \varepsilon'_n$ for all $a\in[k_1], b\in[k_2], z\in \mathcal{Z}$.  By applying a Taylor series expansion, we then have \begin{align}
\left|\log\frac{\R_{aa'}(1\!-\!\widehat{\R}_{aa'})}{\widehat{\R}_{aa'}(1\!-\!\R_{aa'})}\right| \le 2 \varepsilon_n,  \ \left|\log \frac{Q_{ab}(\UU_{ij})}{\widehat{Q}_{ab}(\UU_{ij})} \right| \le 2 \varepsilon'_n. \label{eq:ha3}
\end{align}
Combining~\eqref{eq:ha},~\eqref{eq:ha2}, and~\eqref{eq:ha3}, we complete the proof of Lemma~\ref{claim:con2}.   
\ifCLASSOPTIONcaptionsoff
  \newpage
\fi

\bibliographystyle{IEEEtran}
\bibliography{reference}

\begin{thebibliography}{10}
\providecommand{\url}[1]{#1}
\csname url@samestyle\endcsname
\providecommand{\newblock}{\relax}
\providecommand{\bibinfo}[2]{#2}
\providecommand{\BIBentrySTDinterwordspacing}{\spaceskip=0pt\relax}
\providecommand{\BIBentryALTinterwordstretchfactor}{4}
\providecommand{\BIBentryALTinterwordspacing}{\spaceskip=\fontdimen2\font plus
\BIBentryALTinterwordstretchfactor\fontdimen3\font minus
  \fontdimen4\font\relax}
\providecommand{\BIBforeignlanguage}[2]{{%
\expandafter\ifx\csname l@#1\endcsname\relax
\typeout{** WARNING: IEEEtran.bst: No hyphenation pattern has been}%
\typeout{** loaded for the language `#1'. Using the pattern for}%
\typeout{** the default language instead.}%
\else
\language=\csname l@#1\endcsname
\fi
#2}}
\providecommand{\BIBdecl}{\relax}
\BIBdecl

\bibitem{aggarwal1999horting}
C.~C. Aggarwal, J.~L. Wolf, K.-L. Wu, and P.~S. Yu, ``Horting hatches an egg: A
  new graph-theoretic approach to collaborative filtering,'' in
  \emph{Proceedings of the fifth ACM SIGKDD international conference on
  Knowledge discovery and data mining}, 1999, pp. 201--212.

\bibitem{massa2007trust}
P.~Massa and P.~Avesani, ``Trust-aware recommender systems,'' in
  \emph{Proceedings of the 2007 ACM conference on Recommender systems}, 2007,
  pp. 17--24.

\bibitem{ma2011recommender}
H.~Ma, D.~Zhou, C.~Liu, M.~R. Lyu, and I.~King, ``Recommender systems with
  social regularization,'' in \emph{Proceedings of the ACM international
  conference on Web search and data mining}, 2011, pp. 287--296.

\bibitem{ahn2018binary}
K.~Ahn, K.~Lee, H.~Cha, and C.~Suh, ``Binary rating estimation with graph side
  information,'' in \emph{Advances in Neural Information Processing Systems},
  2018, pp. 4272--4283.

\bibitem{jo2020discrete}
C.~Jo and K.~Lee, ``Discrete-valued latent preference matrix estimation with
  graph side information,'' in \emph{International Conference on Machine
  Learning}, 2021.

\bibitem{kalofolias2014matrix}
V.~Kalofolias, X.~Bresson, M.~Bronstein, and P.~Vandergheynst, ``Matrix
  completion on graphs,'' \emph{NIPS workshop ``Out of the Box: Robustness in
  High Dimension''}, 2014.

\bibitem{goldberg1992using}
D.~Goldberg, D.~Nichols, B.~M. Oki, and D.~Terry, ``Using collaborative
  filtering to weave an information tapestry,'' \emph{Communications of the
  ACM}, vol.~35, no.~12, pp. 61--71, 1992.

\bibitem{sarwar2001item}
B.~Sarwar, G.~Karypis, J.~Konstan, and J.~Riedl, ``Item-based collaborative
  filtering recommendation algorithms,'' in \emph{Proceedings of the 10th
  international conference on World Wide Web}, 2001, pp. 285--295.

\bibitem{condliff1999bayesian}
M.~K. Condliff, D.~D. Lewis, D.~Madigan, and C.~Posse, ``Bayesian mixed-effects
  models for recommender systems,'' in \emph{ACM SIGIR}, vol.~99.\hskip 1em
  plus 0.5em minus 0.4em\relax Citeseer, 1999, pp. 23--30.

\bibitem{rattigan2007graph}
M.~J. Rattigan, M.~Maier, and D.~Jensen, ``Graph clustering with network
  structure indices,'' in \emph{Proceedings of the 24th international
  conference on Machine learning}, 2007, pp. 783--790.

\bibitem{wang2018billion}
J.~Wang, P.~Huang, H.~Zhao, Z.~Zhang, B.~Zhao, and D.~L. Lee, ``Billion-scale
  commodity embedding for e-commerce recommendation in alibaba,'' in
  \emph{Proceedings of the ACM SIGKDD International Conference on Knowledge
  Discovery \& Data Mining}, 2018, pp. 839--848.

\bibitem{eric2021}
Q.~Zhang, V.~Y.~F. Tan, and C.~Suh, ``Community detection and matrix completion
  with social and item similarity graphs,'' \emph{IEEE Transactions on Signal
  Processing}, vol.~69, pp. 917--931, 2021.

\bibitem{bennett2007netflix}
J.~Bennett, S.~Lanning \emph{et~al.}, ``The netflix prize,'' in
  \emph{Proceedings of KDD cup and workshop}, vol. 2007.\hskip 1em plus 0.5em
  minus 0.4em\relax Citeseer, 2007, p.~35.

\bibitem{holland1983stochastic}
P.~W. Holland, K.~B. Laskey, and S.~Leinhardt, ``Stochastic blockmodels: First
  steps,'' \emph{Social networks}, vol.~5, no.~2, pp. 109--137, 1983.

\bibitem{keshavan2010matrix}
R.~H. Keshavan, A.~Montanari, and S.~Oh, ``Matrix completion from a few
  entries,'' \emph{IEEE Trans. Inf. Theory}, vol.~56, pp. 2980--2998, 2010.

\bibitem{jain2013low}
P.~Jain, P.~Netrapalli, and S.~Sanghavi, ``Low-rank matrix completion using
  alternating minimization,'' in \emph{STOC}, 2013, pp. 665--674.

\bibitem{abbe2015community}
E.~Abbe and C.~Sandon, ``Community detection in general stochastic block
  models: Fundamental limits and efficient algorithms for recovery,'' in
  \emph{IEEE 56th Annual Symposium on Foundations of Computer Science}, 2015,
  pp. 670--688.

\bibitem{abbe2017community}
E.~Abbe, ``Community detection and stochastic block models: recent
  developments,'' \emph{The Journal of Machine Learning Research}, vol.~18,
  no.~1, pp. 6446--6531, 2017.

\bibitem{gao2017achieving}
C.~Gao, Z.~Ma, A.~Y. Zhang, and H.~H. Zhou, ``Achieving optimal
  misclassification proportion in stochastic block models,'' \emph{The Journal
  of Machine Learning Research}, vol.~18, no.~1, pp. 1980--2024, 2017.

\bibitem{candes2015phase}
E.~J. Candes, X.~Li, and M.~Soltanolkotabi, ``Phase retrieval via wirtinger
  flow: Theory and algorithms,'' \emph{IEEE Trans. Inf. Theory}, vol.~61,
  no.~4, pp. 1985--2007, 2015.

\bibitem{netrapalli2015phase}
P.~Netrapalli, P.~Jain, and S.~Sanghavi, ``Phase retrieval using alternating
  minimization,'' \emph{IEEE Transactions on Signal Processing}, vol.~63,
  no.~18, pp. 4814--4826, 2015.

\bibitem{rozemberczki2020characteristic}
B.~Rozemberczki and R.~Sarkar, ``Characteristic functions on graphs: Birds of a
  feather, from statistical descriptors to parametric models,'' in
  \emph{Proceedings of the 29th ACM International Conference on Information \&
  Knowledge Management}, 2020, pp. 1325--1334.

\bibitem{adamic2005political}
L.~A. Adamic and N.~Glance, ``The political blogosphere and the 2004 us
  election: divided they blog,'' in \emph{Proceedings of the 3rd international
  workshop on Link discovery}, 2005, pp. 36--43.

\bibitem{candes2010power}
E.~J. Cand{\`e}s and T.~Tao, ``The power of convex relaxation: Near-optimal
  matrix completion,'' \emph{IEEE Transactions on Information Theory}, vol.~56,
  no.~5, pp. 2053--2080, 2010.

\bibitem{marjanovic2012l_q}
G.~Marjanovic and V.~Solo, ``On $ l_q$ optimization and matrix completion,''
  \emph{IEEE Transactions on Signal Processing}, vol.~60, pp. 5714--5724, 2012.

\bibitem{chen2015signal}
S.~Chen, A.~Sandryhaila, J.~M. Moura, and J.~Kova{\v{c}}evi{\'c}, ``Signal
  recovery on graphs: Variation minimization,'' \emph{IEEE Transactions on
  Signal Processing}, vol.~63, no.~17, pp. 4609--4624, 2015.

\bibitem{dai2011subspace}
W.~Dai, O.~Milenkovic, and E.~Kerman, ``Subspace evolution and transfer (set)
  for low-rank matrix completion,'' \emph{IEEE Transactions on Signal
  Processing}, vol.~59, no.~7, pp. 3120--3132, 2011.

\bibitem{ma2014decomposition}
R.~Ma, N.~Barzigar, A.~Roozgard, and S.~Cheng, ``Decomposition approach for
  low-rank matrix completion and its applications,'' \emph{IEEE Transactions on
  Signal Processing}, vol.~62, no.~7, pp. 1671--1683, 2014.

\bibitem{monti2017geometric}
F.~Monti, M.~M. Bronstein, and X.~Bresson, ``Geometric matrix completion with
  recurrent multi-graph neural networks,'' \emph{arXiv preprint
  arXiv:1704.06803}, 2017.

\bibitem{wang2019neural}
X.~Wang, X.~He, M.~Wang, F.~Feng, and T.-S. Chua, ``Neural graph collaborative
  filtering,'' in \emph{Proceedings of the 42nd international ACM SIGIR
  conference on Research and development in Information Retrieval}, 2019, pp.
  165--174.

\bibitem{yoon2018joint}
J.~Yoon, K.~Lee, and C.~Suh, ``On the joint recovery of community structure and
  community features,'' in \emph{56th Annual Allerton Conference on
  Communication, Control, and Computing}, 2018, pp. 686--694.

\bibitem{abbe2015exact}
E.~Abbe, A.~S. Bandeira, and G.~Hall, ``Exact recovery in the stochastic block
  model,'' \emph{IEEE Trans. Inf. Theory}, vol.~62, pp. 471--487, 2015.

\bibitem{mossel2015reconstruction}
E.~Mossel, J.~Neeman, and A.~Sly, ``Reconstruction and estimation in the
  planted partition model,'' \emph{Probability Theory and Related Fields}, vol.
  162, no. 3-4, pp. 431--461, 2015.

\bibitem{zhang2016minimax}
A.~Y. Zhang and H.~H. Zhou, ``Minimax rates of community detection in
  stochastic block models,'' \emph{The Annals of Statistics}, vol.~44, no.~5,
  pp. 2252--2280, 2016.

\bibitem{hajek2017information}
B.~Hajek, Y.~Wu, and J.~Xu, ``Information limits for recovering a hidden
  community,'' \emph{IEEE Trans. Inf. Theory}, vol.~63, pp. 4729--4745, 2017.

\bibitem{saad2018community}
H.~Saad and A.~Nosratinia, ``Community detection with side information: Exact
  recovery under the stochastic block model,'' \emph{IEEE Journal of Selected
  Topics in Signal Processing}, vol.~12, no.~5, pp. 944--958, 2018.

\bibitem{saad2018exact}
------, ``Exact recovery in community detection with continuous-valued side
  information,'' \emph{IEEE Signal Processing Letters}, vol.~26, no.~2, pp.
  332--336, 2018.

\bibitem{mayya2019mutual}
V.~Mayya and G.~Reeves, ``Mutual information in community detection with
  covariate information and correlated networks,'' in \emph{57th Annual
  Allerton Conference on Communication, Control, and Computing}, 2019, pp.
  602--607.

\bibitem{asadi2017compressing}
A.~R. Asadi, E.~Abbe, and S.~Verd{\'u}, ``Compressing data on graphs with
  clusters,'' in \emph{IEEE Int. Symp. Inf. Theory (ISIT)}, 2017, pp.
  1583--1587.

\bibitem{mcpherson2001birds}
M.~McPherson, L.~Smith-Lovin, and J.~M. Cook, ``Birds of a feather: Homophily
  in social networks,'' \emph{Annual review of sociology}, vol.~27, no.~1, pp.
  415--444, 2001.

\bibitem{chin2015stochastic}
P.~Chin, A.~Rao, and V.~Vu, ``Stochastic block model and community detection in
  sparse graphs: A spectral algorithm with optimal rate of recovery,'' in
  \emph{Conference on Learning Theory}, 2015, pp. 391--423.

\bibitem{chen2016community}
Y.~Chen, G.~Kamath, C.~Suh, and D.~Tse, ``Community recovery in graphs with
  locality,'' in \emph{International Conference on Machine Learning}, 2016, pp.
  689--698.

\bibitem{lei2015consistency}
J.~Lei and A.~Rinaldo, ``Consistency of spectral clustering in stochastic block
  models,'' \emph{The Annals of Statistics}, vol.~43, pp. 215--237, 2015.

\bibitem{javanmard2016phase}
A.~Javanmard, A.~Montanari, and F.~Ricci-Tersenghi, ``Phase transitions in
  semidefinite relaxations,'' \emph{Proceedings of the National Academy of
  Sciences}, vol. 113, no.~16, pp. E2218--E2223, 2016.

\bibitem{mossel2016density}
E.~Mossel and J.~Xu, ``Density evolution in the degree-correlated stochastic
  block model,'' in \emph{Conference on Learning Theory}, 2016, pp. 1319--1356.

\bibitem{krzakala2013spectral}
F.~Krzakala, C.~Moore, E.~Mossel, J.~Neeman, A.~Sly, L.~Zdeborov{\'a}, and
  P.~Zhang, ``Spectral redemption in clustering sparse networks,''
  \emph{Proceedings of the National Academy of Sciences}, vol. 110, no.~52, pp.
  20\,935--20\,940, 2013.

\bibitem{NIPS2016_a8849b05}
S.-Y. Yun and A.~Proutiere, ``Optimal cluster recovery in the labeled
  stochastic block model,'' in \emph{Advances in Neural Information Processing
  Systems (NIPS)}, vol.~29, 2016.

\bibitem{halko2011finding}
N.~Halko, P.-G. Martinsson, and J.~A. Tropp, ``Finding structure with
  randomness: Probabilistic algorithms for constructing approximate matrix
  decompositions,'' \emph{SIAM review}, vol.~53, no.~2, pp. 217--288, 2011.

\bibitem{wolfowitz2012coding}
J.~Wolfowitz, \emph{Coding theorems of information theory}.\hskip 1em plus
  0.5em minus 0.4em\relax Springer Science \& Business Media, 2012, vol.~31.

\bibitem{koren2008factorization}
Y.~Koren, ``Factorization meets the neighborhood: a multifaceted collaborative
  filtering model,'' in \emph{Proceedings of the 14th ACM SIGKDD international
  conference on Knowledge discovery and data mining}, 2008, pp. 426--434.

\bibitem{guo2015trustsvd}
G.~Guo, J.~Zhang, and N.~Yorke-Smith, ``Trust{SVD}: Collaborative filtering
  with both the explicit and implicit influence of user trust and of item
  ratings,'' in \emph{AAAI Conference on Artificial Intelligence}, 2015.

\end{thebibliography}

\end{document}